\documentclass[nohyperref]{article}

\usepackage{microtype}
\usepackage{graphicx}
\usepackage{subfigure}
\usepackage{booktabs} 

\usepackage[accepted]{icml2023}
\usepackage{tikzit}

\tikzstyle{empty circle}=[fill=white, draw=black, shape=circle]
\tikzstyle{new style 0}=[fill=none, draw=black, shape=rectangle]
\tikzstyle{new style 1}=[fill=white, draw=none, shape=rectangle]
\tikzstyle{Xrec}=[fill=none, draw=black, shape=rectangle, dashed, minimum width=2cm, minimum height=3.5cm]
\tikzstyle{Zrec}=[fill=none, draw=black, shape=rectangle, dashed, minimum width=2.5cm, minimum height=3.5cm]
\tikzstyle{Yrec}=[fill=none, draw=black, shape=rectangle, dashed, minimum width=1.5cm, minimum height=3.5cm]

\tikzstyle{directed edge}=[fill=none, draw=black, ->, >=angle 60]
\tikzstyle{new edge style 0}=[-, dashed]

\icmltitlerunning{New Rules for Causal Identification with Background Knowledge}

\begin{document}

\twocolumn[
\icmltitle{New Rules for Causal Identification with Background Knowledge}



\icmlsetsymbol{equal}{*}
\begin{icmlauthorlist}
\icmlauthor{Tian-Zuo Wang}{yyy}
\icmlauthor{Lue Tao}{yyy}
\icmlauthor{Zhi-Hua Zhou}{yyy}
\end{icmlauthorlist}

\icmlaffiliation{yyy}{National Key Laboratory for Novel Software Technology, Nanjing University, Nanjing, 210023, China}
\icmlcorrespondingauthor{Zhi-Hua Zhou}{zhouzh@lamda.nju.edu.cn}

\icmlkeywords{Machine Learning, ICML}

\vskip 0.3in
]



\printAffiliationsAndNotice{} 

\begin{abstract}
Identifying causal relations is crucial for a variety of downstream tasks. In additional to observational data, \emph{background knowledge (BK)}, which could be attained from human expertise or experiments, is usually introduced for uncovering causal relations. This raises an open problem that in the presence of latent variables, what causal relations are identifiable from observational data and BK. In this paper, we propose two novel rules for incorporating BK, which offer a new perspective to the open problem. In addition, we show that these rules are applicable in some typical causality tasks, such as determining the set of possible causal effects with observational data. Our rule-based approach enhances the state-of-the-art method by circumventing a process of \emph{enumerating block sets} that would otherwise take exponential complexity. 
\end{abstract}
\section{Introduction}
\label{sec:Introduction}
In recent years, the adoption of causal thinking~\citep{books/2009causality} has opened up new venues for many machine learning topics, such as semi-supervised learning~\citep{conf/icml/ScholkopfJPSZM12,conf/uai/KugelgenMLS20}, reinforcement learning~\citep{conf/iclr/HuangFLM022,conf/iclr/RuanZDB23}, transfer learning~\citep{conf/icml/GongZLTGS16,conf/nips/0001GSHLG20,conf/aaai/CaiC0CZYLYZ21}, and so on. One essence of causal thinking lies in the \emph{causal relations} among the variables, generally characterized by a \emph{causal graph}. As a causal graph is usually not pre-known, uncovering the causal relations is vital for addressing downstream tasks. 

Given observational data, the existing theoretical results have shown that only a Markov equivalence class (MEC) of causal graphs is identifiable, which contains some uncertain causal relations~\citep{books/spirtes2000causation,ali2005orientation,journals/ai/Zhang08}. To further reveal these relations, additional structural knowledge is usually incorporated, which could be attained from experiments or human expertise~\citep{conf/uai/Meek95}. In the literature, we usually call this kind of knowledge by \emph{background knowledge}, or BK for short. 

In the presence of both observational data and BK, a core problem is \emph{causal identification}, \emph{i.e.}, understanding what causal relations are \emph{identifiable} from these knowledge. This problem is vital because it concerns the extent to which causal relations can be inferred from available information. On one hand, it pursues to identify as many causal relations as possible from existing knowledge, which can provide supports for fully utilizing BK in practical tasks~\citep{journals/he2008active,journals/ijar/HauserB14}. On the other hand, even without BK, it is still valuable to some tasks with only observational data, such as causal effect estimation~\citep{journals/2009estimating,conf/uai/FangH20,conf/icml/WangQZ23} and equivalent causal graph enumeration~\citep{conf/aaai/WienobstLBL23,conf/icml/WangTZ24}. In these tasks, some additional structures that can be viewed as hypothetical BK are possibly introduced, thus solving causal identification can facilitate uncovering the most informative causal graph given these additional structures.

Significant efforts have been made towards causal identification in scenarios without latent variables~\citep{conf/uai/VermaP90,conf/uai/VermaP92}. And~\citet{conf/uai/Meek95} closed the problem by presenting four \emph{sound} and \emph{complete} rules to uncover causal relations. However, in real world tasks, \emph{latent confounders} that influence some observable variables generally exist. In these instances, \emph{ancestral graph} is usually used to characterize the causal relations among observable variables~\citep{richardson2002ancestral}. To identify causal relations with observational data in such contexts, ten sound and complete rules have been proposed~\citep{ali2005orientation,journals/ai/Zhang08}. And there are also several studies about causal identification with some kind of \emph{specific} BK~\citep{conf/aistats/Andrews20,conf/nips/JaberKSB20,journals/arXiv/WangQZ2022}. Nevertheless, the thorough result for causal identification with \emph{any} kinds of BK in the presence of latent confounders remains elusive. 


In this paper, we propose two novel rules for incorporating BK in the presence of latent confounders. Different from existing rules which identify causal relations based on few \emph{edges} or \emph{paths}, the identified causal relations by our proposed rules may rely on a \emph{subgraph}. Our findings suggest that in the presence of latent variables, causal identification requires more complicated orientation rules when BK is incorporated, thereby highlighting the intrinsic hardness of causal identification with BK. Interestingly, we find that the proposed rules are essentially the generalizations of two existing rules in the literature. We believe that the proposed rules can inspire the establishment of sound and complete rules to incorporate BK in the future. 


Further, even without BK, the proposed rules are also applicable in some typical causality tasks with only observational data. We show that our proposed rules can take benefit to \emph{set determination} task by improving the state-of-the-art method PAGcauses~\citep{conf/icml/WangQZ23}. As previously discussed, with observational data, we can only identify a MEC, within which the causal effect of a variable $X$ on variable $Y$ is possibly unidentifiable. To mitigate this unidentifiable case, a common solution is to determine the \emph{set of possible causal effects} instead, which consists of the causal effect values in all the causal graphs within the MEC, which is called \emph{set determination} for brevity. In the absence of latent variables, many efficient methods have been proposed for set determination~\citep{journals/2009estimating,conf/uai/PerkovicKM17,conf/uai/FangH20,witte2020efficient}. For the scenarios with latent confounders,~\citet{journal/malinsky2016} proposed the first relevant method by locally enumerating MAGs. Then,~\citet{journals/arXiv/WangQZ2022} presented an enumeration-free method \emph{PAGcauses}, which reduces the complexity super-exponentially compared to the enumeration-based method. In this paper, we introduce the proposed rules to enhance PAGcauses by avoiding a process of enumerating block sets, which reduces an exponential complexity relative to the number of vertices.

In summary, this paper makes two significant contributions. Firstly, we present two novel rules for incorporating BK in the presence of latent confounders. Secondly, we apply the rules in set determination task, effectively eliminating an exponential computational burden of the state-of-the-art method. All the proofs are shown in appendix.
\section{Preliminary}
\label{sec:Preliminary}
Denote a graph by $G$. Let $\mf{V}(G)$ denote the set of vertices (variables) and $\mf{E}(G)$ denote the set of edges in $G$. We use bold letter (\emph{e.g.}, $\mf{A}$) to denote a set of vertices and normal letter (\emph{e.g.}, $A$) to denote a vertex. Given a set of vertices $\mf{V}'\subseteq \mf{V}(G)$, $G[\mf{V}']$ is the subgraph of $G$ induced by $\mf{V}'$ which consists of vertices $\mf{V}'$ and all the edges between $\mf{V}'$. $G[-\mf{V}']$ denotes $G[\mf{V}(G)\backslash \mf{V}']$. $G$ is a \emph{complete graph} if for any two vertices in $G$, there is an edge connecting them.

In this paper, we assume the absence of selection bias. Hence the case for selection bias is not involved in the following definitions. A graph is a \emph{mixed graph} if it contains directed and bi-directed edges. The two ends of an edge are \emph{marks}, which could be \emph{arrowhead}, \emph{tail}, and \emph{circle}($\circ$). The symbol $\circ$ represents that the mark here is unknown. The symbol $\ast$ is a \emph{wildcard} that represents any marks. A \emph{partial mixed graph (PMG)} is a graph containing arrowheads, tails, and circles. Due to space limit, some definitions are shown in Appendix~\ref{sec: prelimiary about graphs}, including \emph{directed path, minimal path, collider path, parent, ancestor, descendant, possible ancestor, possible descendant, circle edge, circle component}.


In a graph $G$, if there is $V_i\rightarrowast V_j\leftarrowast V_k$ where $V_i$ is not adjacent to $V_k$, they form an \emph{unshielded collider}. Consider a path $p=\langle V_1,V_2,\cdots,V_k\rangle$, $p$ is a \emph{possible directed path} if for the edge between $V_i$ and $V_{i+1}$, $\forall 1\leq i\leq k-1$, there is no arrowhead at $V_i$ and no tail at $V_{i+1}$; $p$ is \emph{uncovered} if $V_{i-1}$ is not adjacent to $V_{i+1}$, $\forall 2\leq i\leq k-1$. In $G$, denote the set of parents/ancestors/descendants/possible descendants of $V_i$ by $\text{Pa}(V_i,G)/\text{Anc}(V_i,G)/\text{De}(V_i,G)/\text{PossDe}(V_i,G)$. Given a vertex $V_i$ and a set of vertices $\mf{V}'$ in $G$, $V_i\in \text{Anc}(\mf{V}',G)$ if there exists a vertex $V_j\in \mf{V}'$ such that $V_i\in \text{Anc}(V_j,G)$. 

For a mixed graph $G$, if there is a directed path from $V_i$ to $V_j$ and an edge $V_j\rightarrow V_i$/$V_j\leftrightarrow V_i$, they form a \emph{directed cycle/almost directed cycle}. A mixed graph $G$ is \emph{ancestral} if there are no directed cycles and no almost directed cycles. The \emph{maximal property} is given in Appendix~\ref{sec: prelimiary about graphs}. In the presence of latent variables, \emph{maximal ancestral graph (MAG)} is usually used to characterize the causal relations among observable variables. Essentially, MAG is a projection graph on the observable variables of an underlying DAG that contains both observable and latent variables. We say \emph{a DAG $\ml{D}$ is represented by a MAG $\ml{M}$} if $\ml{M}$ is a projection graph of an underlying $\ml{D}$. Note many DAGs can be represented by one MAG, which is detailed in Appendix~\ref{sec: prelimiary about graphs}. A \emph{partial ancestral graph (PAG)} represents a Markov equivalence class (MEC) of MAGs. Denote MAG and PAG by $\ml{M}$ and $\ml{P}$, respectively. Suppose we obtain a PMG $H$ from $\ml{P}$ by transforming some circles. We say a MAG $\ml{M}$ is \emph{consistent with} $H$ if $\ml{M}$ has the same non-circle marks at $H$ and $\ml{M}$ belongs to the MEC represented by $\ml{P}$. Note when we say an edge $\rightcircleast$, the $\ast$ here is not a tail, for otherwise the circle can only be an arrowhead due to no selection bias. $G_{\utilde{X}}$ denotes the subgraph of $G$ by deleting all the edges out of $X$.

In the literature, there are orientation rules $\ml{R}_1-\ml{R}_{11}$ to identify a PAG or incorporate \emph{local background knowledge} into a PAG. These rules are shown in Appendix~\ref{subsec:preliminary about rules}.

\emph{Covariate adjustment} is a classical method to estimate the causal effect given a causal graph, by finding an \emph{adjustment set} $\mf{Z}$ such that $P(Y|do(X))=\sum_\mf{Z} P(\mf{Z})P(Y|X,\mf{Z})\diff \mf{Z}$. More related results are shown in Appendix~\ref{subsec:preliminary about causal effect}.
\begin{figure*}[tb]
  \centering
  \subfigure[]{\label{figure:11}
  \includegraphics[height=0.18\linewidth]{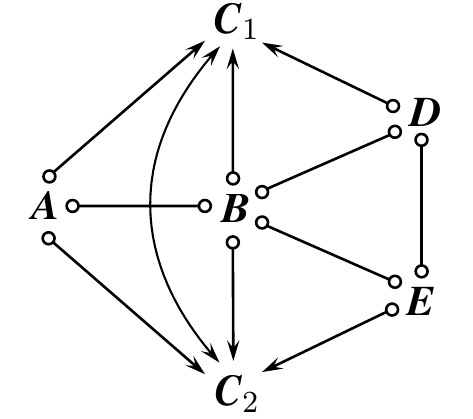}} \hspace{0.24cm}
  \subfigure[]{\label{figure:12}
  \includegraphics[height=0.18\linewidth]{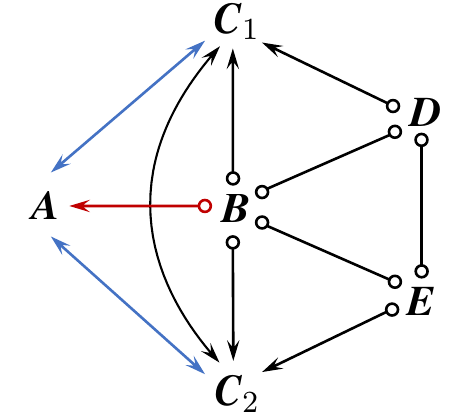}} \hspace{0.24cm}
  \subfigure[]{\label{figure:13}
  \includegraphics[height=0.18\linewidth]{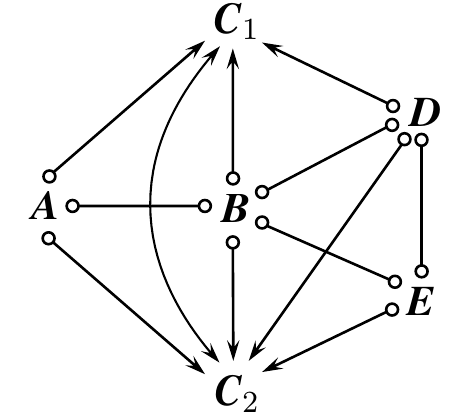}} \hspace{0.24cm}
  \subfigure[]{\label{figure:14}
  \includegraphics[height=0.18\linewidth]{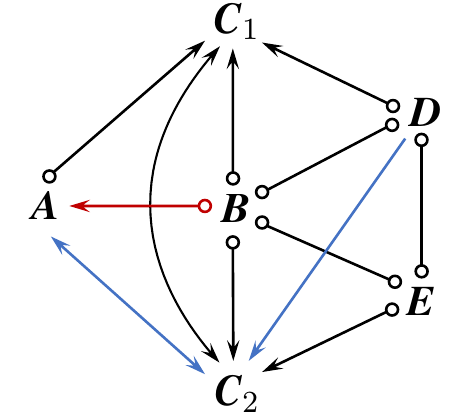}} \hspace{0.24cm}
  \caption{Two examples for $\ml{R}_{12}$ and $\ml{R}_{13}$. Two PAGs are shown in Fig.~\ref{figure:11} and~\ref{figure:13}. Blue lines denote the edges transformed according to BK, red lines denote the edges transformed by $\ml{R}_{12}$ and $\ml{R}_{13}$. Fig.~\ref{figure:12} shows a PMG transformed from Fig.~\ref{figure:11} with additional BK and $\ml{R}_{12}$. Fig.~\ref{figure:14} shows a PMG transformed from Fig.~\ref{figure:13} with additional BK and $\ml{R}_{13}$.} 
  \label{figure:1} 
  \end{figure*}
\section{Proposed Rules}
\label{sec:proposed rule}
In this section, we present two novel rules to incorporate BK into a partial mixed graph (PMG) $H$. As there have established sound and complete rules to obtain a PAG with observational data in the literature~\citep{journals/ai/Zhang08}, we do not consider the stage of identifying a PAG in this paper. Hence, we restrict that $H$ is a PAG or a PMG transformed from a PAG. Also, we assume that the introduced BK is correct, \emph{i.e.}, there exist MAGs consistent with the PMG and BK.

At first, we introduce an important concept, \emph{unbridged path relative to $\mf{V}'$} in a PMG $H$, in Def.~\ref{def:unbridged path}, where $\mf{V}'$ is a set of vertices in $H$. Intuitively, an unbridged path $p$ relative to $\mf{V}'$ is a path with an \emph{intriguing property}: if every vertex in $p$ is not an ancestor of $\mf{V}'$ in $H$, then every vertex in $p$ \emph{must be} an ancestor of $\mf{V}'$ in any MAG consistent with $H$.
\begin{myDef}[Unbridged path relative to $\mf{V}'$]
\label{def:unbridged path}
Suppose $H$ a PMG. If there is an uncovered circle path $p:V_0\leftrightcircle V_1\leftrightcircle\cdots \leftrightcircle V_n,n\geq 1$ in $H[-\mf{V}']$ such that $\ml{F}_{V_0}\backslash \ml{F}_{V_{1}}\neq \emptyset$ and $\ml{F}_{V_n}\backslash \ml{F}_{V_{n-1}}\neq \emptyset$, where $\ml{F}_{V_i}=\{V\in \mf{V}'\mid  V\rightcircleast V_i \mbox{ or } V\rightarrowast V_i\mbox{ in }H\}$, then $p$ is an unbridged path relative to $\mf{V}'$.
\end{myDef}
\begin{myRem}
One may wonder why the abovementioned property holds for unbridged path $p$ if every vertex in $p$ is not an ancestor of $\mf{V}'$ in $H$. The reason is, in any MAG $\ml{M}$ consistent with $H$, there cannot be additional unshielded colliders relative to $H$, which introduce additional conditional independence such that the graphs do not belong to the MEC. Suppose $C_1\in \ml{F}_{V_0}\backslash \ml{F}_{V_1}$ and $C_2\in \ml{F}_{V_n}\backslash \ml{F}_{V_{n-1}}$ according to Def.~\ref{def:unbridged path}. Since (1) $C_1\not\in \ml{F}_{V_1}$ and (2) $V_1$ is not an ancestor of $C_1\in\mf{V}'$ in $H$, we can conclude that $C_1$ is not adjacent to $V_1$. Similarly, $C_2$ is not adjacent to $V_{n-1}$. Hence, to avoid generating unshielded colliders, the corresponding path in $\ml{M}$ of $p$ as well as $C_1$ and $C_2$ can only be $C_1\leftrightast V_0\rightarrow \cdots\rightarrow V_n\rightarrow C_2$, $C_1\leftarrow V_0\leftarrow \cdots \leftarrow V_n\leftrightast C_2$, or $C_1\leftarrow V_0\leftarrow \cdots \leftarrowast$ $V_i\rightarrow \cdots V_n \rightarrow C_2$. In any case, any vertex in $p$ is an ancestor of either $C_1$ or $C_2$. See Fig.~\ref{figure:11} for an example. $D\leftrightcircle E$ is an unbridged path relative to $\mf{V}'=\{C_1,C_2\}$ due to $C_1\in \ml{F}_{D}\backslash \ml{F}_E$ and $C_2\in \ml{F}_E\backslash \ml{F}_D$. If we transform all the circles in $C_1\leftarrowcircle D\leftrightcircle E\rightarrowcircle C_2$ without generating unshielded colliders, $D$ and $E$ must be ancestors of either $C_1$ or $C_2$.
\end{myRem}
Next, we present the orientation rule $\ml{R}_{12}$ inspired by the property above, and then the orientation rule $\ml{R}_{13}$ as a supplement of the case of $\ml{R}_{12}$ when some vertex in the unbridged path \emph{has been} an ancestor of $\mf{V}'$ in $H$.\footnote{Recently, $\ml{R}_{13}$ was independently identified by~\citet{venkateswaran2024towards}, along with some other fundamental results.}
\begin{enumerate}
\item[$\ml{R}_{12}$] Suppose an edge $A\leftcircleast B$ in a PMG $H$. Let $\mf{S}_A=\{V\in\mf{V}(H)|V\rightarrowast A\mbox{ in }H\}\cup\{A\}$. If there is an unbridged path $\langle K_1,\cdots,K_m\rangle$ relative to $\mf{S}_A$ in $H[-\mf{S}_A]$ and for every vertex $K_i\in \{K_1,\cdots,K_m\}$, there exists an uncovered possible directed path $\langle A,B,\cdots,K_i \rangle$ ($B\neq K_i$), then orient $A\leftcircleast B$ as $A\leftarrowast B$.
\item[$\ml{R}_{13}$] Suppose an edge $A\leftcircleast B$ in a PMG $H$. Let $\mf{S}_A=\{V\in\mf{V}(H)|V\rightarrowast A\mbox{ in }H\}\cup\{A\}$. If there is an uncovered possible directed path $\langle A,B,\cdots,K\rangle$ in $H$, where $K\in \mbox{\normalfont Anc}(\mf{S}_A,H)$, then orient $A\leftcircleast B$ as $A\leftarrowast B$.
\end{enumerate}

We present two examples for $\ml{R}_{12}$ and $\ml{R}_{13}$ in Fig.~\ref{figure:1}. Consider PAG in Fig.~\ref{figure:11} and BK $C_1\rightarrowast A\leftarrowast$ $C_2$ in Fig.~\ref{figure:12}. See $\ml{R}_{12}$, $\mf{S}_A=\{C_1,C_2,A\}$, and there exist uncovered possible directed paths $p_1=\langle A,B,D\rangle$ and $p_2=\langle A,B,E\rangle$ from $A$ to $D$ and $E$, respectively. $D\leftrightcircle E$ is unbridged relative to $\mf{S}_A$. Hence, the edge between $A$ and $B$ is transformed to $A\leftarrowcircle B$ by $\ml{R}_{12}$. This transformation is intuitive after knowing the property of unbridged path. In Fig.~\ref{figure:12}, no vertex in the unbridged path is an ancestor of $\mf{S}_A$, thus $D$ and $E$ are ancestors of $\mf{S}_A$ in any MAG consistent with $H$. Without loss of generality, suppose $D$ is ancestor of $C_1$. Due to the uncovered possible directed path $p_1$, if there is $A\rightarrow B$, $p_1$ can only be a directed path from $A$ to $D$, and thus there is an almost directed cycle $A\rightarrow B\rightarrow D\rightarrow C_1\leftrightarrow A$, which violates the ancestral property. For $\ml{R}_{13}$, see a PAG in Fig.~\ref{figure:13}. If BK is $C_2\leftrightarrow A$ and $D\rightarrow C_2$ as Fig.~\ref{figure:14}, there is $\mf{S}_A=\{A,C_2\}$ and an uncovered possible directed path $\langle A,B,D\rangle$ where $D\in \text{Anc}(\mf{S}_A,H)$. Hence we transform $A\leftrightcircle B$ to $A\leftarrowcircle B$ for the same reason as above.

We present Thm.~\ref{Thm:rule 12 all vertices} to imply the soundness of $\ml{R}_{12}$ and $\ml{R}_{13}$ to incorporate BK in the presence of latent confounders. Note previous rules~\citep{journals/ai/Zhang08,conf/aistats/Andrews20,journals/arXiv/WangQZ2022} cannot trigger these two transformations. Recently, ~\citet{venkateswaran2024towards} independently discover $\ml{R}_{13}$, along with some fundamental results, while $\ml{R}_{12}$ is not involved. 
\begin{algorithm}[tb]
  \KwIn{PMG $H$}
  \KwOut{Updated $H$}
  \While{there is an edge $A\leftcircleast B$ in $H$}{
  Obtain $\mf{S}_A=\{V\in \mf{V}(H)|V\rightarrowast A\mbox{ in }H\}\cup\{A\}$\;
  Obtain a set of vertices $\mf{D}$ defined as $V\in\mf{D}$ if and only if $V\in \mf{V}(H)\backslash \mf{S}_A$ and there is an uncovered path $p$ from $A$ to $V$ where $B$ is the vertex adjacent to $A$ in $p$\;
  \uIf{there exists $V\in\mf{D}$ such that $V\in{\normalfont \mbox{Anc}}(\mf{S}_A,H)$}{Transform $A\leftcircleast B$ to $A\leftarrowast B$}
  \Else{
  Obtain graph $H'$ based on $H$ by transforming $V\leftcircleast V'$ to $V\leftarrowast V'$, $\forall V\in \mf{D},\forall V'\in\mf{S}_A$\;
  Update the circle component in $H'[\mf{D}]$ as follows until no updates: for $V_i,V_j\in \mf{D}$, transform $V_i\leftrightcircle V_j$ into $V_i\rightarrow V_j$ if either of the two conditions holds (1) $\ml{F}_{V_i}\backslash \ml{F}_{V_j}\neq \emptyset$; or (2) there is a vertex $V_k\in \mf{D}$ such that there is $V_k\rightarrow V_i$ and $V_k$ is not adjacent to $V_j$, where $\ml{F}_V=\{V'\in \mf{S}_A |V'\rightcircleast V\mbox{ or }V'\rightarrowast V \mbox{ in }H\}$\;
  \lIf{there are new unshielded colliders in $H'$}{Transform $A\leftcircleast B$ to $A\leftarrowast B$ in $H$}}}
  \caption{Implementation of $\ml{R}_{12}$ and $\ml{R}_{13}$}
  \label{alg: R1213}
\end{algorithm}

\begin{myThm}
\label{Thm:rule 12 all vertices}
$\ml{R}_{12}$ and $\ml{R}_{13}$ are sound to incorporate BK.
\end{myThm}
One remaining issue is the implementation of $\ml{R}_{12}$ and $\ml{R}_{13}$. We provide Alg.~\ref{alg: R1213} to implement $\ml{R}_{12}$ and $\ml{R}_{13}$,\footnote{The main focus of our paper is not on implementation. It is possible that there can be a more efficient method to implement $\ml{R}_{12}$ and $\ml{R}_{13}$, which we leave for future work.} with theoretical guarantee in Prop.~\ref{prop:implementation soundess}. Essentially, the edges transformed on Line 5 of Alg.~\ref{alg: R1213} are triggered by $\ml{R}_{13}$, the edges transformed on Line 9 are triggered by $\ml{R}_{12}$. Line 8 involves detecting the presence of unbridged paths, during which new unshielded colliders are generated if such paths are found. See the proof of Prop.~\ref{alg: R1213} for the details. Suppose the number of edges in $H$ is $m$, the complexity of implementing Alg.~\ref{alg: R1213} is $\ml{O}(m^3)$, detailed in Appendix~\ref{sec:complexity of rules}. 
\begin{prop}
\label{prop:implementation soundess}
Given a PMG $H$, Alg.~\ref{alg: R1213} can transform all and only the edges that can be transformed by $\ml{R}_{12}$ or $\ml{R}_{13}$.
\end{prop}
Note $\ml{R}_{12}$ is quite different from the previous rules $\ml{R}_1-\ml{R}_{11}$, which are shown in Appendix~\ref{subsec:preliminary about rules}. These existing rules transform an edge based on just few edges or few paths. $\ml{R}_{12}$ is more complicated, since $\ml{R}_{12}$ considers not only the unbridged path, but also the large number of paths from $A$ in $\ml{R}_{12}$ to every vertex in the unbridged path, which form a sub-graph. The establishment of this rule implies the intrinsic hardness of causal identification from observational data and BK in the presence of latent variables.

Interestingly, we find that essentially, $\ml{R}_{12}$ and $\ml{R}_{13}$ are two \emph{generalizations} of $\ml{R}_3$ and $\ml{R}_2$. See $\ml{R}_2$ and $\ml{R}_{3}$ in Appendix~\ref{subsec:preliminary about rules}. Suppose a PMG $H$. Consider there is $E\rightarrow C\rightarrowast A$: $\ml{R}_2$ says if there is an edge $A\leftcircleast E$, then we orient it as $A\leftarrowast E$; while $\ml{R}_{13}$ says if there is an uncovered possible directed path from $A$ to $E$ beginning with $A\leftcircleast B$, then we orient $A\leftcircleast B$ as $A\leftarrowast B$. $\ml{R}_{13}$ generalizes an edge $A\leftcircleast E$ in $\ml{R}_2$ to an uncovered possible directed path from $A$ to $E$ beginning with $A\leftcircleast B$. $\ml{R}_{12}$ is also a generalization of $\ml{R}_3$. Consider there is an unshielded triple $C\rightarrowast A\leftarrowast D$ in a PMG $H$: $\ml{R}_3$ says if there is $C\rightcircleast B\leftcircleast D$ and an edge $A\leftcircleast B$, then we orient $A\leftcircleast B$ to $A\leftarrowast B$. Here the reason for the transformation is, although $B$ is not an ancestor of $\{C,D\}$ in $H$, $B$ must be an ancestor of either $C$ or $D$ in any MAG consistent with $H$. That is, the vertex $B$ here has the same property as the unbridged path we discuss above. Consider $C\rightarrowast A$ and $D\rightarrowast A$ in $\ml{R}_3$, if there is $A\rightarrow B$, either $A,B,C$ or $A,B,D$ will form a directed or almost directed cycle, which violates the ancestral property. While $\ml{R}_{12}$ says if there is an unbridged path $p$ relative to $\{A,C,D\}$ and $B$ is the vertex adjacent to $A$ in the uncovered path from $A$ to every vertex in $p$, then we orient $A\leftcircleast B$ to $A\leftarrowast B$. $\ml{R}_{13}$ generalizes the one vertex $B$ in $\ml{R}_3$ to a subgraph induced by $B$ and all the vertices $V$ such that there is an uncovered possible directed path from $A$ to $V$ beginning with $A\leftcircleast B$.
\section{Application on Set Determination}
\label{sec:proposed method}
In this section, we demonstrate the applicability of the rules proposed in Sec.~\ref{sec:proposed rule} to causality tasks that rely solely on observational data. Specifically, we focus on \emph{set determination} task in the presence of latent confounders, \textit{i.e.}, determining the set of possible causal effects of vertex $X$ on vertex $Y$ with observational data. We will present a rule-based method by introducing the proposed rules into the state-of-the-art method PAGcauses~\citep{conf/icml/WangQZ23}, which takes a substantial improvement on efficiency. 
\begin{figure*}[tb]
  \centering
  \subfigure[$\ml{P}$]{\label{figure:21}
  \includegraphics[height=0.16\linewidth]{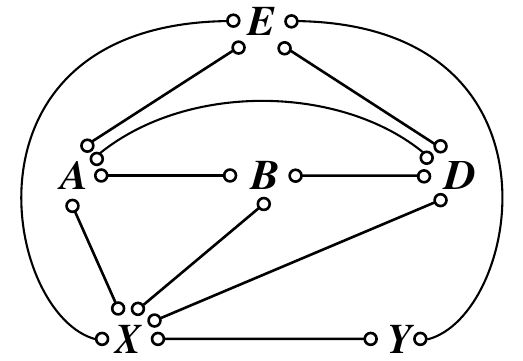}} \hspace{-0.15cm}
  \subfigure[$\mb{M}_1$]{\label{figure:22}
  \includegraphics[height=0.16\linewidth]{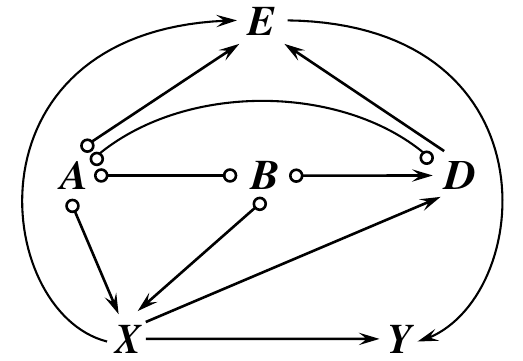}} \hspace{-0.15cm}
  \subfigure[$\mb{M}_2$]{\label{figure:23}
  \includegraphics[height=0.16\linewidth]{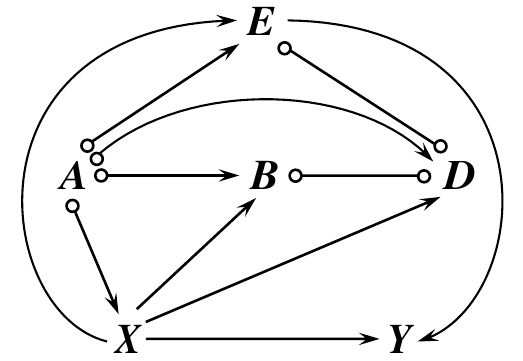}} \hspace{-0.15cm}
  \subfigure[$\mb{M}'$]{\label{figure:24}
  \includegraphics[height=0.16\linewidth]{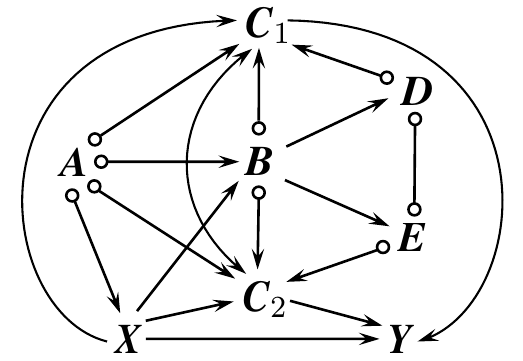}} \hspace{-0.15cm}
  \caption{Fig.~\ref{figure:21} shows a PAG $\ml{P}$. In the first step of PAGcauses, they obtain the maximal local MAG by introducing the local transformation of $X$ and updating the graph with their proposed orientation rules. $\mb{M}_1$ and $\mb{M}_2$ in Fig.~\ref{figure:22} and~\ref{figure:23} are two examples of the maximal local MAG obtained after different local transformations of $X$. Fig.~\ref{figure:23} and~\ref{figure:24} show two examples for implementing Alg.~\ref{alg:get S} given $\mf{W}=\emptyset$, where $\mf{S}=\{E\}$ and $\mf{S}=\{C_1,C_2,B\}$ are returned, respectively.}
  \label{figure:2} 
  \end{figure*}

In set determination task, we have a PAG $\ml{P}$ identified with observational data~\citep{books/spirtes2000causation}. Suppose we are interested in the causal effect of $X$ on $Y$ henceforth. As many causal graphs are consistent with $\ml{P}$ and possibly associated with different causal effects, the target of set determination is to determine the set of \emph{possible causal effects}, consisting of all the causal effects in the DAGs represented by the MAGs in the MEC represented by $\ml{P}$. Following previous studies~\citep{journals/2009estimating,journal/malinsky2016,conf/uai/FangH20,conf/icml/WangQZ23}, we only consider the possible causal effects that are estimated by covariate adjustment, and only focus on \emph{finding all the adjustment sets} for estimating the possible causal effects. Note there possibly exist some DAGs with causal effects unidentifiable by adjustment, all the relevant methods cannot output these effects as it is beyond the ability of covariate adjustment.

We start by revisiting PAGcauses in Sec.~\ref{subsec: reviewing latest results}. Note PAGcauses is a complicated method, which is hard to introduce in limited space. Hence, we just show the core idea of PAGcauses and the parts where $\ml{R}_{12}$ can take an improvement. The details are given in Appendix~\ref{sec:More Related Results}. Then, we propose the theoretical supports for set determination based on $\ml{R}_{12}$ in Sec.~\ref{subsec:untilizing proposed rules}, and present the rule-based algorithm in Sec.~\ref{subsec:method}.

\subsection{Revisiting Latest Results}
\label{subsec: reviewing latest results}
\emph{PAGcauses} is a two-step method. See PAG $\ml{P}$ in Fig.~\ref{figure:21} for an example. As there are some circles at $X$, the first step is to transform all the circles at $X$. This step is similar to the classical method IDA which applies for set determination without latent variables~\citep{journals/2009estimating}. For each possible transformations of (the circles at) $X$, PAGcauses introduces a graphical characterization to evaluate the \emph{validity} of each transformation of $X$, that is, whether there is a MAG consistent with the local transformation in the MEC represented by $\ml{P}$. For each valid transformation, they obtain an updated PMG $\mb{M}$ called \emph{maximal local MAG} with the proposed sound and complete orientation rules. $\mb{M}_1$ and $\mb{M}_2$ in Fig.~\ref{figure:22},\ref{figure:23} are two examples of maximal local MAGs, obtained from different local transformation of $X$. 

In the presence of latent variables, obtaining the maximal local MAG $\mb{M}$ by incorporating valid local transformation of $X$ is not sufficient for determining the causal effect of $X$ on $Y$ in all the MAGs consistent with $\mb{M}$, which is quite different from the case absence of latent confounders~\citep{journals/2009estimating}. Hence, they have to further consider the uncertain structures in $\mb{M}$. However, enumerating all the MAGs and then finding adjustment sets take a super-exponential complexity, which is evidently infeasible. Hence, in the second step of \emph{PAGcauses}, for each $\mb{M}$, \citet{conf/icml/WangQZ23} established the graphical characterization in Prop.~\ref{prop:existencecritical} to \emph{directly} evaluate whether each given set of vertices $\mf{W}$ can be an adjustment set in some MAG consistent with $\mb{M}$. With this result, instead of enumerating the super-exponential number of MAGs, PAGcauses only needs to enumerate exponential number of vertices set $\mf{W}$ whose space is super-exponentially less than that of MAGs, and evaluates whether each $\mf{W}$ is an adjustment set in some MAG by Prop.~\ref{prop:existencecritical}. Before presenting Prop.~\ref{prop:existencecritical}, we first show two definitions. The definition of \emph{bridged} in Prop.~\ref{prop:existencecritical} is presented in Appendix~\ref{sec:More Related Results}.
\begin{myDef}
\label{def:T1 and T2}
Given a set of vertices $\mf{W}$ in a maximal local MAG $\mb{M}$, define a set of vertices $\bar{\mf{W}}$ as $V\in \bar{\mf{W}}$ if and only if $V\in \text{PossAn}(Y,\mb{M})\backslash \mf{W}$ and there exists a collider path beginning with an arrowhead from $X$ to $V$ where each non-endpoint vertex belongs to $\mf{W}$. Denote ${\normalfont \mbox{Anc}}(Y\cup \mf{W},\mb{M})\cap[{\normalfont \mbox{PossDe}}(\bar{\mf{W}},\mb{M})\backslash \bar{\mf{W}}]$ by $\mf{S}_{min}$ and ${\normalfont \mbox{PossAn}}(Y\cup \mf{W},\mb{M})\cap[{\normalfont \mbox{PossDe}}(\bar{\mf{W}},\mb{M})\backslash \bar{\mf{W}}]$ by $\mf{S}_{max}$. $\mf{S}$ is a \emph{block set} if $\mf{S}_{min} \subseteq \mf{S} \subseteq \mf{S}_{max}$.
\end{myDef}
\begin{myDef}
\label{def: possible adjust}
In a maximal local MAG $\mb{M}$, $\mf{W}$ is a \emph{potential adjustment set} if 
\begin{enumerate}
\item[(1)] $\forall V\in \mf{W}$, there is a collider path $X\leftrightarrow \cdots \leftarrowast V$ such that each non-endpoint belongs to $\mf{W}$, and there is a possible directed path from $V$ to $Y$ that does not go through the vertices in $\bar{\mf{W}}$;
\item[(2)] $\mf{W}\cap {\normalfont \mbox{PossDe}}(X,\mb{M})=\emptyset$;
\item[(3)] $\bar{\mf{W}}\cap {\normalfont \mbox{Anc}}(Y\cup \mf{W},\mb{M})=\emptyset$.
\end{enumerate}
\end{myDef}
\begin{prop}
\label{prop:existencecritical}
Given a maximal local MAG $\mb{M}$, for any potential adjustment set $\mf{W}$, there exists a MAG $\ml{M}$ consistent with $\mb{M}$ such that $\mf{W}$ is an adjustment set in $\ml{M}$ if there exists a block set $\mf{S}$ such that
\begin{enumerate}
\item[(1)] $\mbox{\normalfont PossDe}(\bar{\mf{W}},\mb{M}[-\mf{S}])\cap \mbox{\normalfont Pa}(\mf{S},\mb{M})=\emptyset$;
\item[(2)] $\mb{M}[\mf{S}_{V}]$ is a complete graph for any $V\in \bar{\mf{W}}$, where $\mf{S}_{V}=\{V'\in\mf{S}|V\leftcircleast V'\mbox{ in }\mb{M}\}$; 
\item[(3)] $\mb{M}[\mbox{\normalfont PossDe}(\bar{\mf{W}},\mb{M}[-\mf{S}])]$ is bridged relative to $\mf{S}$ in $\mb{M}$.
\end{enumerate}
\end{prop}
Prop.~\ref{prop:existencecritical} provides a sufficient condition for the existence of MAGs consistent with $\mb{M}$ such that a given set $\mf{W}$ is an adjustment set. When using this condition in PAGcauses, it includes (1) enumerate every block set $\mf{S}$; and (2) given each $\mf{S}$, evaluate the three conditions. The complexity of (2) is $\ml{O}(d^3)$, where $d$ denotes the number of vertices. The intuition of the sufficient condition is, they try to construct a MAG $\ml{M}$ such that the adjustment set is $\mf{W}$. To ensure it, they have to restrict that some vertices (characterized by set $\bar{\mf{W}}$) are not ancestors of $\mf{W}\cup\{Y\}$ in $\ml{M}$. Hence, it is necessary to introduce some additional arrowheads to \emph{block} all the possible directed paths from $\bar{\mf{W}}$ to $\mf{W}\cup\{Y\}$. For this purpose, intuitively, the block set $\mf{S}$ in Def.~\ref{def:T1 and T2} is introduced to characterize the position to introduce arrowheads to \emph{block} the possibly directed paths (See Alg.~\ref{algo:maximal localMAG to MAG} in Appendix~\ref{sec:More Related Results} for MAG construction algorithm). There are exponential number (with respect to $d$) of ways to introduce arrowheads to achieve it, hence there are exponential number of block sets in Def.~\ref{def:T1 and T2}. Thus, given $\mf{W}$, evaluating the existence of MAGs in Prop.~\ref{prop:existencecritical} needs to enumerate every block set $\mf{S}$. Essentially, the process of enumerating $\mf{S}$ and evaluating the three conditions is to determine the existence of a kind of way to introduce arrowheads such that a MAG can be constructed with $\mf{W}$ being the adjustment set.
\subsection{Utilizing Proposed Rules}
\label{subsec:untilizing proposed rules}
As discussed in Sec.~\ref{subsec: reviewing latest results}, given a potential adjustment set $\mf{W}$, an exponential complexity of enumerating all the block sets is involved in using Prop.~\ref{prop:existencecritical}, because there are exponential number of possible ways to block all the possible directed paths from $\bar{\mf{W}}$ to $\mf{W}\cup\{Y\}$ such that no vertex in $\bar{\mf{W}}$ is an ancestor of $\mf{W}\cup\{Y\}$ in the constructed MAG. In the following, we show that enumerating block sets can be avoided by introducing $\ml{R}_{12}$, thus can circumvent the possibly exponential computational burden here.

Note the final target is to find a way to introduce arrowheads to prevent $\bar{\mf{W}}$ being ancestors of $\mf{W}\cup\{Y\}$. At first we determine a set of vertices $\mf{S}_0$ in Def.~\ref{def:S_V for V in barW}.
\begin{myDef}
\label{def:S_V for V in barW}
Given a maximal local MAG $\mb{M}$ and a potential adjustment set $\mf{W}$, we define $\mf{S}_0$ as $V'\in \mf{S}_0$ if and only if there exists a vertex $V\in \bar{\mf{W}}$ and there exists a minimal possible directed path $p$ from $V$ to ${\normalfont \mbox{Anc}}(\mf{W}\cup\{Y\},\mb{M})$ such that $V'$ is the vertex adjacent to $V$ in $p$ and each non-endpoint in $p$ does not belong to $\bar{\mf{W}}$.
\end{myDef}
It is evident that if we want to obtain a MAG $\ml{M}$ consistent with $\mb{M}$ such that $\bar{\mf{W}}$ are not ancestors of $\mf{W}\cup\{Y\}$ in $\ml{M}$, all the edges $V\leftcircleast S,V\in \bar{\mf{W}},S\in \mf{S}_0$ in $\mb{M}$ must be transformed to $V\leftarrowast S$, for otherwise there will be a directed path from $V$ to $\mf{W}\cup\{Y\}$. Initialize $\mf{S}=\mf{S}_0$, characterizing the positions to introduce arrowheads, similar to the block set in Prop.~\ref{prop:existencecritical}. The transformed edges above are hypothesis BK, and can possibly trigger $\ml{R}_{12}$ to further update the graph. The updates triggered by $\ml{R}_{12}$ help us enlarge the set $\mf{S}$.

Following the idea above, we present Alg.~\ref{alg:get S}. Note when we propose $\ml{R}_{12}$, the premise is that BK is correct. However, the premise does not necessarily hold here. Hence, in the process of triggering $\ml{R}_{12}$, we need to evaluate whether the hypothesis BK is valid, which is achieved on Line 4 and 12. When $\ml{R}_{12}$ can no longer be triggered (Line 8), roughly speaking, if there is not an unbridged path relative to $\mf{S}$, $\mf{S}$ is a block set that satisfies the three conditions of Prop.~\ref{prop:existencecritical}.
\begin{algorithm}[tb]
  \KwIn{Maximal local MAG $\mb{M}$, $X$, $Y$, $\mf{W}$}
  \KwOut{$\mf{S}$}
  $\mf{S}$ is initialized as $\mf{S}_0$ in Def.~\ref{def:S_V for V in barW}, and $\bar{\mf{W}}$ is determined as Def.~\ref{def:T1 and T2}\;
  $\mf{T}={\normalfont \mbox{PossDe}}(\bar{\mf{W}},\mb{M}[-\mf{S}])\backslash \bar{\mf{W}}$\;
  \While{1}{
  \lIf{$\mb{M}[\mf{S}_V]$ is not a complete graph for some $V\in \bar{\mf{W}}$, where $\mf{S}_V=\{V'\in \mf{S}|V\leftcircleast V'\mbox{or }V\leftarrowast V'\mbox{ in }\mb{M}\}$; or ${\normalfont \mbox{PossDe}}(\bar{\mf{W}},\mb{M}[-\mf{S}])\cap {\normalfont \mbox{Pa}}(\mf{S},\mb{M})\not=\emptyset$}{\Return ``No''}
  Update $\mb{M}$ by transforming $V\leftcircleast S$ to $V\leftarrowast S$ for any $V\in \bar{\mf{W}}$ and $S\in \mf{S}$\;
  \uIf{an edge $A\leftarrowast B$ can be transformed by $\ml{R}_{12}$ in $\mb{M}$}{$\mf{S}=\mf{S}\cup ({\normalfont \mbox{Anc}}(B,\mb{M})\cap \mf{T})$}
  \Else{\uIf{there is not an unbridged path relative to $\mf{S}$ in $\mb{M}[{\normalfont \mbox{PossDe}}(\bar{\mf{W}},\mb{M}[-\mf{S}])]$}{\Return $\mf{S}$}
  \Else{\Return ``No''}}}
  \caption{Updating $\mf{S}$}
  \label{alg:get S}
\end{algorithm}

Then, we present relevant theoretical guarantees for Alg.~\ref{alg:get S}. Thm.~\ref{prop:existencecritical rules} implies that given a maximal local MAG $\mb{M}$ and a potential adjustment set $\mf{W}$, if we can obtain a set by Alg.~\ref{alg:get S}, then there is a MAG $\ml{M}$ consistent with $\mb{M}$ such that $\mf{W}$ is an adjustment set in $\ml{M}$. According to Thm.~\ref{prop:existencecritical rules}, whether Alg.~\ref{alg:get S} outputs a set of vertices or ``No'' is an indicator of whether the input $\mf{W}$ is an adjustment set in some MAG or not. Hence, by executing Alg.~\ref{alg:get S} for each potential adjustment set $\mf{W}$, we can find a set of adjustment sets in the MAGs consistent with $\mb{M}$. To further ensure that we can estimate all the causal effects identifiable by covariate adjustment, we provide Thm.~\ref{thm:existenceDSEP rules}, to indicate that if there exists a MAG $\ml{M}$ consistent with $\mb{M}$ such that the causal effect is identifiable by covariate adjustment, then there is an adjustment set in $\ml{M}$ being a potential adjustment set such that Alg.~\ref{alg:get S} can return a set of vertices. It ensures that via using Alg.~\ref{alg:get S} for each potential adjustment set $\mf{W}$, we can estimate all the possible causal effects that are identifiable by adjustment.
\begin{myThm}
\label{prop:existencecritical rules}
Given a maximal local MAG $\mb{M}$, for any potential adjustment set $\mf{W}$, if Alg.~\ref{alg:get S} could return a set of vertices $\mf{S}$, then there exists a MAG $\ml{M}$ consistent with $\mb{M}$ such that $\mf{W}$ is an adjustment set in $\ml{M}$.
\end{myThm}
\begin{myThm}
\label{thm:existenceDSEP rules}
Given a maximal local MAG $\mb{M}$, suppose a MAG $\ml{M}$ consistent with $\mb{M}$ such that there exists an adjustment set relative to $(X,Y)$. Let $\mf{W}$ be ${\normalfont \mbox{D-SEP}}(X,Y,\ml{M}_{\utilde{X}})$ as Def.~\ref{def:d-sep}. Then $\mf{W}$ is a potential adjustment set in $\mb{M}$ and Alg.~\ref{alg:get S} can return a set of vertices $\mf{S}$ given $\mf{W}$ as the input.
\end{myThm}
\begin{myRem}
There are at most $d(d-1)/2$ circles that could be transformed on Line 6 of Alg.~\ref{alg:get S}, where $d$ denotes the number of vertices. And the transformation on Line 6 is a necessary condition for entering the next round of loop. Hence, the number of loop in Alg.~\ref{alg:get S} is at most $\ml{O}(d^2)$. For the other parts in Alg~\ref{alg:get S}, the complexity is at most $\ml{O}(d^3)$. Hence Alg.~\ref{alg:get S} can be implemented in polynomial time.
\end{myRem}
\begin{myRem}
One may wonder given a maximal local MAG $\mb{M}$ and a potential adjustment set $\mf{W}$, whether we could determine whether $\mf{W}$ can be an adjustment set in some MAG consistent with $\mb{M}$ by judging the three conditions in Prop.~\ref{prop:existencecritical} with a prescribed $\mf{S}$, such as $\mf{S}=\mf{S}_{min}$ or $\mf{S}=\mf{S}_{max}$ in Def.~\ref{def:T1 and T2}. It is infeasible. Consider $\mb{M}_2$ in Fig.~\ref{figure:23} and $\mf{W}=\emptyset$, the three conditions (in Prop.~\ref{prop:existencecritical}) hold when $\mf{S}=\mf{S}_{min}=\{E\}$, but do not hold when $\mf{S}=\mf{S}_{max}=\{B,D,E\}$. While consider $\mb{M}'$ in Fig.~\ref{figure:24} and $\mf{W}=\emptyset$, the three conditions do not hold when $\mf{S}=\mf{S}_{min}=\{C_1,C_2\}$, but hold when $\mf{S}=\{C_1,C_2,B\}$. Hence, it is not direct that which block set can satisfy the three conditions. Intuitively, the benefit taken by utilizing $\ml{R}_{12}$ is that it implies which vertex should be added into $\mf{S}$, instead of enumerating all block sets as Prop.~\ref{prop:existencecritical}. For example, when using Alg.~\ref{alg:get S} for $\mb{M}_2$ in Fig.~\ref{figure:23} given $\mf{W}=\emptyset$, there is $\mf{S}_0=\{E\}$ and no unbridged paths, thus $\mf{S}=\{E\}$ is returned; while for $\mb{M}'$ in Fig.~\ref{figure:24} given $\mf{W}=\emptyset$, there is $\mf{S}_0=\{C_1,C_2\}$ and an unbridged path $D\leftrightcircle E$, thus $\mf{S}=\{C_1,C_2,B\}$ is returned.
\end{myRem}
\subsection{The Algorithm for Set Determination}
\label{subsec:method}
In this part, we present the improved algorithm for set determination, through utilizing the proposed rules. Given a PAG $\ml{P}$, we can obtain several maximal local MAGs $\mb{M}$ based on different local transformations of $X$. Then in each $\mb{M}$, a direct method is to find each subset $\mf{W}\subseteq \mf{V}(\mb{M})\backslash \{X,Y\}$, and then evaluate whether it is a potential adjustment set. If it is a potential adjustment set, we further evaluate whether a set of vertices can be returned by Alg.~\ref{alg:get S}. By this direct method, we can find all the adjustment sets according to Thm.~\ref{prop:existencecritical rules} and Thm.~\ref{thm:existenceDSEP rules}.

However, the method above is still somewhat inefficient, because the enumeration of all the subsets of $\mf{V}(\mb{M})\backslash \{X,Y\}$ is not always necessary. In fact, given a maximal local MAG, some vertices can be determined to belong to the adjustment set ${\normalfont \mbox{D-SEP}}(X,Y,\ml{M}_{\utilde{X}})$ in any MAG $\ml{M}$ (See Appendix~\ref{subsec:preliminary about causal effect} for ${\normalfont \mbox{D-SEP}}(X,Y,\ml{M}_{\utilde{X}})$). In the following, we present Def.~\ref{def:dd-sep} to characterize these vertices, and show ${\normalfont \mbox{DD-SEP}}(X,Y,\mb{M}_{\utilde{X}})\subseteq {\normalfont \mbox{D-SEP}}(X,Y,\ml{M}_{\utilde{X}})$ in Prop.~\ref{prop: DD-SEP}. 
\begin{myDef}[${\normalfont \mbox{DD-SEP}}(X,Y,\mb{M}_{\utilde{X}})$]
\label{def:dd-sep}
Let $\mb{M}$ be a maximal local MAG. $V\in {\normalfont \mbox{DD-SEP}}(X,Y,\mb{M}_{\utilde{X}})$ if and only if there is a collider path $X\leftrightarrow V_1\leftrightarrow \cdots \leftrightarrow  V_{k-1} \leftarrowast V$ in $\mb{M}_{\utilde{X}}$, where $(1)$ $Y\in \text{PossDe}(X,\mb{M})$; $(2)$ $V_1,\cdots,V_{k-1}\in {\normalfont \mbox{DD-SEP}}(X,Y,\mb{M}_{\utilde{X}})$; $(3)$ $V\in \text{Anc}(Y,\mb{M})$ or the subgraph $\mb{M}[\ml{Q}_V]$ is not a complete graph, where $\ml{Q}_V=\{V'\in \text{Anc}(Y,\mb{M})\mid V\leftcircleast V'\mbox{ in }\mb{M}\}$.
\end{myDef}

\begin{algorithm}[tb]
  \KwIn{PAG $\ml{P}$, $X$, $Y$}
  \KwOut{$\widehat{\mbox{\normalfont AS}}(\ml{P})$\Comment{Adjustment sets in MAGs consistent with $\ml{P}$}}
  $\widehat{\mbox{\normalfont AS}}(\ml{P})= \emptyset$ \Comment{Record all the valid adjustment sets}\;
  \lIf{$X\not\in{\normalfont \mbox{PossAn}}(Y,\ml{P})$}{\Return No causal effects}
  \uIf{the conditions in Prop.~\ref{prop: improved generalized backdoor} are satisfied for $\ml{P}$}{\Return $\widehat{\mbox{\normalfont AS}}(\ml{P})\leftarrow \{{\normalfont \mbox{D-SEP}}(X,Y,\ml{P}_{\underline{X}})\}$\Comment{Prop.~\ref{prop: improved generalized backdoor}}}
  \For{each set $\mf{C}\subseteq \{V\mid V\rightcircleast X\mbox{ in }\ml{P}\}$}{
      \uIf{the three conditions in Prop.~\ref{thm:existence} are satisfied}{
      Obtain a maximal local MAG $\mb{M}$ based on $\ml{P}$ and $\mf{C}$\;
      Find all potential adjustment sets $\mf{W}_1,\mf{W}_2,\cdots,$ that contains ${\normalfont \mbox{DD-SEP}}(X,Y,\mb{M})$ given $\mb{M}$\;
      \For{each potential adjustment set $\mf{W}_i$}{
        Obtain $\mf{S}_0$ and $\bar{\mf{W}}$ as Def.~\ref{def:S_V for V in barW} and Def.~\ref{def:T1 and T2}\;
        \uIf{Alg.~\ref{alg:get S} can return a set of vertices given $\mf{W}_i$ and $\mb{M}$}{$\widehat{\mbox{\normalfont AS}}(\ml{P})\leftarrow \widehat{\mbox{\normalfont AS}}(\ml{P})\cup \{\mf{W}_i\}$\;}
      }}
  }
  \caption{PAGrules}
  \label{alg:obtain causal bounds with rules}
\end{algorithm}
\begin{prop}
\label{prop: DD-SEP}
Given a maximal local MAG $\mb{M}$, if $V\in {\normalfont \mbox{DD-SEP}}(X,Y,\mb{M}_{\utilde{X}})$, then $V\in {\normalfont \mbox{D-SEP}}(X,Y,\ml{M}_{\utilde{X}})$ in any MAG $\ml{M}$ consistent with $\mb{M}$ such that there exists an adjustment set relative to $(X,Y)$. 
\end{prop}
Combining the parts above, we provide \emph{PAGrules} in Alg.~\ref{alg:obtain causal bounds with rules} for set determination, based on the framework of PAGcauses. Note if $X\not\in \text{PossAn}(Y,\ml{P})$ on Line 2, $X$ has no causal effect on $Y$ in all the MAGs consistent with $\ml{P}$, thus there is no need to consider the set of possible causal effects. Similarly, if the causal effect is identifiable by covariate adjustment in $\ml{P}$, we can directly return the causal effect by Prop.~\ref{prop: improved generalized backdoor}, which is detailed in Appendix~\ref{subsec:preliminary about causal effect}. On Line 5, we enumerate all possible local transformations of $X$, which is represented by a set of vertices $\mf{C}$ that implies transforming the edge $X\leftcircleast V$ to $X\leftarrowast V$ if $V\in \mf{C}$ and to $X\rightarrow V$ if $V\not\in \mf{C}$. Evidently, all the sets $\mf{C}$ can represent all the possible local transformations of $X$. On Line 6-7, we obtain all the valid local transformations of $X$ and obtain the maximal local MAGs based on each valid local transformation. According to Prop.~\ref{prop: DD-SEP}, we only need to consider the potential adjustment set containing ${\normalfont \mbox{DD-SEP}}(X,Y,\mb{M})$, which is executed on Line 8. And from Line 9-12, we evaluate whether each potential adjustment set is an adjustment set by Alg.~\ref{alg:get S}.

Finally, Cor.~\ref{cor:corsetequals} implies that PAGrules can return the set of causal effects that are identifiable by covariate adjustment in the DAGs represented by the MAGs consistent with $\ml{P}$.
\begin{corollary}
\label{cor:corsetequals}
For a PAG $\ml{P}$, denote $\text{CE}(\ml{P})$ and $\widehat{\text{CE}}(\ml{P})$ the set of causal effects in the DAGs represented by the MAGs consistent with $\ml{P}$ that can be estimated by covariate adjustment with observable variables and the set of causal effects output by Alg.~\ref{alg:obtain causal bounds with rules}. It holds that $\text{CE}(\ml{P})\stackrel{set}{=}\widehat{\text{CE}}(\ml{P})$. 
\end{corollary}

\section{Conclusion}
\label{sec:conclusion}
It remains an open problem for a long time that what causal relations are identifiable from observational data and background knowledge (BK) in the presence of latent variables. In this paper, we propose two novel rules for incorporating BK. The rules are quite different from the existing rules in the form. We believe the proposed rules can take some new insights to the open problem and could possibly inspire the establishment of the sound and complete rules in the future. Further, by utilizing the proposed rules, we present a novel algorithm to determine the set of possible causal effects given a PAG. The rules can help improve the efficiency of the state-of-the-art method.

In the future, it is worthy to continue investigating the sound and complete rules to incorporate BK, based on the rules that have been established. In addition, in light of the fact that only a Markov equivalence class (MEC) is identifiable given observational data, the causal effect is usually unidentifiable within a MEC. Hence, exploring and exploiting the additional knowledge that is available in practice may be helpful to causal effect estimation tasks. 


\bibliography{multisetICML24}

\begin{thebibliography}{32}
\providecommand{\natexlab}[1]{#1}
\providecommand{\url}[1]{\texttt{#1}}
\expandafter\ifx\csname urlstyle\endcsname\relax
  \providecommand{\doi}[1]{doi: #1}\else
  \providecommand{\doi}{doi: \begingroup \urlstyle{rm}\Url}\fi

\bibitem[Ali et~al.(2005)Ali, Richardson, Spirtes, and Zhang]{ali2005orientation}
Ali, R.~A., Richardson, T., Spirtes, P., and Zhang, J.
\newblock Orientation rules for constructing markov equivalence classes of maximal ancestral graphs.
\newblock Technical report, 2005.

\bibitem[Andrews et~al.(2020)Andrews, Spirtes, and Cooper]{conf/aistats/Andrews20}
Andrews, B., Spirtes, P., and Cooper, G.~F.
\newblock On the completeness of causal discovery in the presence of latent confounding with tiered background knowledge.
\newblock In \emph{Proceedings of the 23rd International Conference on Artificial Intelligence and Statistics}, pp.\  4002--4011, 2020.

\bibitem[Cai et~al.(2021)Cai, Chen, Li, Chen, Zhang, Ye, Li, Yang, and Zhang]{conf/aaai/CaiC0CZYLYZ21}
Cai, R., Chen, J., Li, Z., Chen, W., Zhang, K., Ye, J., Li, Z., Yang, X., and Zhang, Z.
\newblock Time series domain adaptation via sparse associative structure alignment.
\newblock In \emph{Proceedings of the 35th {AAAI} Conference on Artificial Intelligence}, pp.\  6859--6867, 2021.

\bibitem[Fang \& He(2020)Fang and He]{conf/uai/FangH20}
Fang, Z. and He, Y.
\newblock {IDA} with background knowledge.
\newblock In \emph{Proceedings of the Thirty-Sixth Conference on Uncertainty in Artificial Intelligence}, volume 124 of \emph{Proceedings of Machine Learning Research}, pp.\  270--279. {AUAI} Press, 2020.

\bibitem[Gong et~al.(2016)Gong, Zhang, Liu, Tao, Glymour, and Sch{\"{o}}lkopf]{conf/icml/GongZLTGS16}
Gong, M., Zhang, K., Liu, T., Tao, D., Glymour, C., and Sch{\"{o}}lkopf, B.
\newblock Domain adaptation with conditional transferable components.
\newblock In Balcan, M. and Weinberger, K.~Q. (eds.), \emph{Proceedings of the 33nd International Conference on Machine Learning}, pp.\  2839--2848, 2016.

\bibitem[Hauser \& B{\"{u}}hlmann(2014)Hauser and B{\"{u}}hlmann]{journals/ijar/HauserB14}
Hauser, A. and B{\"{u}}hlmann, P.
\newblock Two optimal strategies for active learning of causal models from interventional data.
\newblock \emph{International Journal of Approximate Reasoning}, 55\penalty0 (4):\penalty0 926--939, 2014.

\bibitem[He \& Geng(2008)He and Geng]{journals/he2008active}
He, Y. and Geng, Z.
\newblock Active learning of causal networks with intervention experiments and optimal designs.
\newblock \emph{Journal of Machine Learning Research}, 9:\penalty0 2523--2547, 2008.

\bibitem[Huang et~al.(2022)Huang, Feng, Lu, Magliacane, and Zhang]{conf/iclr/HuangFLM022}
Huang, B., Feng, F., Lu, C., Magliacane, S., and Zhang, K.
\newblock Adarl: What, where, and how to adapt in transfer reinforcement learning.
\newblock In \emph{Proceedings of the 10th International Conference on Learning Representations}, 2022.

\bibitem[Jaber et~al.(2020)Jaber, Kocaoglu, Shanmugam, and Bareinboim]{conf/nips/JaberKSB20}
Jaber, A., Kocaoglu, M., Shanmugam, K., and Bareinboim, E.
\newblock Causal discovery from soft interventions with unknown targets: Characterization and learning.
\newblock In \emph{Advances in Neural Information Processing Systems}, 2020.

\bibitem[Maathuis et~al.(2009)Maathuis, Kalisch, B{\"u}hlmann, et~al.]{journals/2009estimating}
Maathuis, M.~H., Kalisch, M., B{\"u}hlmann, P., et~al.
\newblock Estimating high-dimensional intervention effects from observational data.
\newblock \emph{The Annals of Statistics}, 37\penalty0 (6A):\penalty0 3133--3164, 2009.

\bibitem[Maathuis et~al.(2015)Maathuis, Colombo, et~al.]{journals/AOS/maathuis2015}
Maathuis, M.~H., Colombo, D., et~al.
\newblock A generalized back-door criterion.
\newblock \emph{The Annals of Statistics}, 43\penalty0 (3):\penalty0 1060--1088, 2015.

\bibitem[Malinsky \& Spirtes(2016)Malinsky and Spirtes]{journal/malinsky2016}
Malinsky, D. and Spirtes, P.
\newblock Estimating causal effects with ancestral graph markov models.
\newblock In \emph{Conference on Probabilistic Graphical Models}, pp.\  299--309, 2016.

\bibitem[Meek(1995)]{conf/uai/Meek95}
Meek, C.
\newblock Causal inference and causal explanation with background knowledge.
\newblock In \emph{Proceedings of the 11st Annual Conference on Uncertainty in Artificial Intelligence}, pp.\  403--410, 1995.

\bibitem[Pearl(2009)]{books/2009causality}
Pearl, J.
\newblock \emph{Causality}.
\newblock Cambridge University Press, 2009.

\bibitem[Perkovic et~al.(2017{\natexlab{a}})Perkovic, Kalisch, and Maathuis]{conf/uai/PerkovicKM17}
Perkovic, E., Kalisch, M., and Maathuis, M.~H.
\newblock Interpreting and using cpdags with background knowledge.
\newblock In \emph{Proceedings of the 33rd Conference on Uncertainty in Artificial Intelligence}, 2017{\natexlab{a}}.

\bibitem[Perkovic et~al.(2017{\natexlab{b}})Perkovic, Textor, Kalisch, and Maathuis]{journals/jmlr/PerkovicTKM17}
Perkovic, E., Textor, J., Kalisch, M., and Maathuis, M.~H.
\newblock Complete graphical characterization and construction of adjustment sets in markov equivalence classes of ancestral graphs.
\newblock \emph{Journal of Machine Learning Research}, 18:\penalty0 220:1--220:62, 2017{\natexlab{b}}.

\bibitem[Richardson et~al.(2002)Richardson, Spirtes, et~al.]{richardson2002ancestral}
Richardson, T., Spirtes, P., et~al.
\newblock Ancestral graph markov models.
\newblock \emph{The Annals of Statistics}, 30\penalty0 (4):\penalty0 962--1030, 2002.

\bibitem[Ruan et~al.(2023)Ruan, Zhang, Di, and Bareinboim]{conf/iclr/RuanZDB23}
Ruan, K., Zhang, J., Di, X., and Bareinboim, E.
\newblock Causal imitation learning via inverse reinforcement learning.
\newblock In \emph{Proceedings of the 11st International Conference on Learning Representations}, 2023.

\bibitem[Sch{\"{o}}lkopf et~al.(2012)Sch{\"{o}}lkopf, Janzing, Peters, Sgouritsa, Zhang, and Mooij]{conf/icml/ScholkopfJPSZM12}
Sch{\"{o}}lkopf, B., Janzing, D., Peters, J., Sgouritsa, E., Zhang, K., and Mooij, J.~M.
\newblock On causal and anticausal learning.
\newblock In \emph{Proceedings of the 29th International Conference on Machine Learning}, pp.\  459--466, 2012.

\bibitem[Spirtes et~al.(2000)Spirtes, Glymour, and Scheines]{books/spirtes2000causation}
Spirtes, P., Glymour, C.~N., and Scheines, R.
\newblock \emph{Causation, prediction, and search}.
\newblock MIT Press, 2000.

\bibitem[van~der Zander et~al.(2014)van~der Zander, Liskiewicz, and Textor]{conf/uai/ZanderLT14a}
van~der Zander, B., Liskiewicz, M., and Textor, J.
\newblock Constructing separators and adjustment sets in ancestral graphs.
\newblock In \emph{Proceedings of the Thirtieth Conference on Uncertainty in Artificial Intelligence}, pp.\  907--916, 2014.

\bibitem[Venkateswaran \& Perkovic(2024)Venkateswaran and Perkovic]{venkateswaran2024towards}
Venkateswaran, A. and Perkovic, E.
\newblock Towards complete causal explanation with expert knowledge.
\newblock \emph{arXiv preprint arXiv:2407.07338}, 2024.

\bibitem[Verma \& Pearl(1990)Verma and Pearl]{conf/uai/VermaP90}
Verma, T. and Pearl, J.
\newblock Equivalence and synthesis of causal models.
\newblock In \emph{Proceedings of the Sixth Annual Conference on Uncertainty in Artificial Intelligence}, pp.\  255--270, 1990.

\bibitem[Verma \& Pearl(1992)Verma and Pearl]{conf/uai/VermaP92}
Verma, T. and Pearl, J.
\newblock An algorithm for deciding if a set of observed independencies has a causal explanation.
\newblock In \emph{Proceedings of the 8th Annual Conference on Uncertainty in Artificial Intelligence}, pp.\  323--330, 1992.

\bibitem[von K{\"{u}}gelgen et~al.(2020)von K{\"{u}}gelgen, Mey, Loog, and Sch{\"{o}}lkopf]{conf/uai/KugelgenMLS20}
von K{\"{u}}gelgen, J., Mey, A., Loog, M., and Sch{\"{o}}lkopf, B.
\newblock Semi-supervised learning, causality, and the conditional cluster assumption.
\newblock In \emph{Proceedings of the 36th Conference on Uncertainty in Artificial Intelligence}, pp.\  1--10, 2020.

\bibitem[Wang et~al.(2023{\natexlab{a}})Wang, Qin, and Zhou]{conf/icml/WangQZ23}
Wang, T.-Z., Qin, T., and Zhou, Z.-H.
\newblock Estimating possible causal effects with latent variables via adjustment.
\newblock In \emph{Proceedings of the 40th International Conference on Machine Learning}, pp.\  36308--36335, 2023{\natexlab{a}}.

\bibitem[Wang et~al.(2023{\natexlab{b}})Wang, Qin, and Zhou]{journals/arXiv/WangQZ2022}
Wang, T.-Z., Qin, T., and Zhou, Z.-H.
\newblock Sound and complete causal identification with latent variables given local background knowledge.
\newblock \emph{Artificial Intelligence}, 322:\penalty0 103964, 2023{\natexlab{b}}.

\bibitem[Wang et~al.(2024)Wang, Du, and Zhou]{conf/icml/WangTZ24}
Wang, T.-Z., Du, W.-B., and Zhou, Z.-H.
\newblock An efficient maximal ancestral graph listing algorithm.
\newblock In \emph{Proceedings of the 41st International Conference on Machine Learning}, 2024.

\bibitem[Wien{\"{o}}bst et~al.(2023)Wien{\"{o}}bst, Luttermann, Bannach, and Liskiewicz]{conf/aaai/WienobstLBL23}
Wien{\"{o}}bst, M., Luttermann, M., Bannach, M., and Liskiewicz, M.
\newblock Efficient enumeration of markov equivalent dags.
\newblock In \emph{Proceedings of the 37th {AAAI} Conference on Artificial Intelligence}, pp.\  12313--12320, 2023.

\bibitem[Witte et~al.(2020)Witte, Henckel, Maathuis, and Didelez]{witte2020efficient}
Witte, J., Henckel, L., Maathuis, M.~H., and Didelez, V.
\newblock On efficient adjustment in causal graphs.
\newblock \emph{The Journal of Machine Learning Research}, 21\penalty0 (1):\penalty0 9956--10000, 2020.

\bibitem[Zhang(2008)]{journals/ai/Zhang08}
Zhang, J.
\newblock On the completeness of orientation rules for causal discovery in the presence of latent confounders and selection bias.
\newblock \emph{Artificial Intelligence}, 172\penalty0 (16-17):\penalty0 1873--1896, 2008.

\bibitem[Zhang et~al.(2020)Zhang, Gong, Stojanov, Huang, Liu, and Glymour]{conf/nips/0001GSHLG20}
Zhang, K., Gong, M., Stojanov, P., Huang, B., Liu, Q., and Glymour, C.
\newblock Domain adaptation as a problem of inference on graphical models.
\newblock In \emph{Advances in Neural Information Processing Systems 33}, 2020.

\end{thebibliography}
\bibliographystyle{icml2022}

\appendix
\onecolumn
\section{Detailed Preliminary}
\label{sec: a detailed preliminary}
\subsection{Preliminary about Graphs}
\label{sec: prelimiary about graphs}
In a graph $G$, consider a path $p=\langle V_1,V_2,\cdots,V_k\rangle$, $p$ is a \emph{directed path} if there is $V_i\rightarrow V_{i+1}$, $\forall 1\leq i\leq k-1$; $p$ is a collider path if there is $V_{i-1}\rightarrowast V_i\leftarrowast V_{i+1}$, $\forall 2\leq i\leq k-1$; $p$ is \emph{minimal} if any two non-consecutive vertices are not adjacent. A vertex $V_1$ is a \emph{parent} of $V_2$ if there is $V_1\rightarrow V_2$. $V_1$ is an \emph{ancestor/descendant} of $V_2$ if there is a directed path from $V_1$/$V_2$ to $V_2$/$V_1$. $V_1$ is a \emph{possible ancestor/possible descendant} of $V_2$ if there is a possible directed path from $V_1$/$V_2$ to $V_2$/$V_1$. Note each vertex is an ancestor/descendant/possible ancestor/possible descendant of itself. An edge in the form of $\leftrightcircle$ is a \emph{circle edge}. The \emph{circle component} of a graph $G$ is the subgraph of $G=(\mf{V},\mf{E})$ consisting of the vertices $\mf{V}$ and all the circle edges. We say two vertices $V_i$ and $V_j$ is a \emph{connected circle component} in $G$ if there is a circle path $V_i\leftrightcircle \cdots\leftrightcircle V_j$ from $V_i$ to $V_j$ in $G$. For two paths $p_1=\langle V_1,V_2,\cdots,V_d\rangle $ and $p_2=\langle S_1,S_2,\cdots,S_n\rangle$, we use $p_1\bigoplus p_2$ to denote the concatenate path $\langle V_1,\cdots,V_d,S_1,\cdots,S_n\rangle$.

If there is a path $V_1\rightarrow V_2\rightarrow \cdots\rightarrow V_d$ and an edge $V_d\rightarrow V_1$/$V_d\leftrightarrow V_1$, then there is a \emph{directed cycle/almost directed cycle}. For a mixed graph, if there is not a directed cycle or almost directed cycle, then it is \emph{ancestral}. For an ancestral graph, if for any two non-adjacent vertices, there is a set of vertices that \emph{m-seperates} them, then the graph is \emph{maximal}. If a mixed graph is both ancestral and maximal, it is a \emph{maximal ancestral graph (MAG)}, denoted by $\ml{M}$. 

Essentially, MAG is a projection graph on the observable variables of some underlying DAGs containing both observable and latent variables.~\citet{books/spirtes2000causation,journals/ai/Zhang08} presented the algorithm to obtain a MAG with vertices $\mf{O}$ from a DAG with vertices $\mf{O}\cup \mf{L}\cup\mf{S}$, where $\mf{O},\mf{L},\mf{S}$ denote the observable vertices set, latent vertices set, and selection vertices set. Since we do not consider selection bias in this paper, we do not consider $\mf{S}$ in the following. Next we show inducing path in Def.~\ref{def:incuding path}, followed by the algorithm to obtain a MAG based on a DAG. According to the algorithm, as the number of latent vertices can be arbitrary, there can be infinite number of DAGs which could lead to one MAG by the algorithm. 

\begin{myDef}[Inducing path;~\citet{books/spirtes2000causation}]
\label{def:incuding path}
Let $X,Y$ be two vertices in an maximal ancestral graph. Denote $\mf{L},\mf{S}$ two disjoint sets of vertices that $X,Y$ do not belong to. A path $p$ from $X$ to $Y$ is an \emph{inducing path relative to $\langle \mf{L},\mf{S}\rangle$} if every non-endpoint vertex on $p$ is either in $\mf{L}$ or a collider, and every collider on $p$ is an ancestor of either $X,Y$, or a member of $\mf{S}$.
\end{myDef}

\textbf{Input:} a DAG $\ml{D}$ over $\mf{V}=\mf{O}\cup \mf{L}$;\\
\textbf{Output:} a MAG $\ml{M}$ over $\mf{O}$.\\
$(1)$ for each pair of vertices $A,B\in \mf{O}$, $A$ and $B$ are adjacent in $\ml{M}$ if and only if there is an inducing path relative to $\langle \mf{L},\emptyset\rangle$ from $A$ to $B$ in $\ml{D}$;\\
$(2)$ for each pair of adjacent vertices $A,B$ in $\ml{M}$, orient the edge between $A$ and $B$ according to the following steps:
\begin{enumerate}
\item[(a)] orient $A\rightarrow B$ in $\ml{M}$ if $A\in {\normalfont\mbox{Anc}}(B,\ml{D})$ and $B\not\in {\normalfont\mbox{Anc}}(A,\ml{D})$;
\item[(b)] orient $A\leftarrow B$ in $\ml{M}$ if $B\in {\normalfont\mbox{Anc}}(A,\ml{D})$ and $A\not\in {\normalfont\mbox{Anc}}(B,\ml{D})$;
\item[(c)] orient $A\leftrightarrow B$ in $\ml{M}$ if $A\not\in {\normalfont\mbox{Anc}}(B,\ml{D})$ and $B\not\in {\normalfont\mbox{Anc}}(A,\ml{D})$.
\end{enumerate}

Given observational data, we can only identify a \emph{Markov equivalence class (MEC)} of MAGs, which is represented by a \emph{partial ancestral graph (PAG)} and denoted by $\ml{P}$. In a PAG, there is an arrowhead/tail at some position if and only if there is an arrowhead/tail at this position in all the MAGs in the MEC; and there is a circle at some position if and only if there are both arrowheads and tails at this position in all the MAGs in the MEC. 

\subsection{Preliminary about Orientation Rules}
\label{subsec:preliminary about rules}
\citet{ali2005orientation,journals/ai/Zhang08} proposed ten rules $\ml{R}_1-\ml{R}_{10}$ to identify a PAG with observational data.~\citet{conf/nips/JaberKSB20} presented the solid result to imply that when there is additional interventional data and selection bias is not allowed for, the ten rules are also complete. Another study~\citep{conf/aistats/Andrews20} indicates that the ten rules are complete if we incorporate \emph{tiered background knowledge}, which means that the BK can classify the variables into distinct parts, where the causal order between different parts is explicit, but the structural information within each part cannot be directly known according to the BK.

Further, when we have obtained a PAG and incorporate \emph{local background knowledge (local BK)}, \emph{i.e.}, the full structural knowledge regarding some specific variables,~\citet{journals/arXiv/WangQZ2022} proposed a rule $\ml{R}_4'$ to replace $\ml{R}_4$ and an additional rule $\ml{R}_{11}$. ~\citet{journals/arXiv/WangQZ2022} proved that the rules are sound and complete to incorporate local BK into a PAG. We show these rules in the following. Since
$\ml{R}_5-\ml{R}_7$ are triggered only if the selection bias is involved and we assume the absence of selection bias, we omit these three rules.

\begin{enumerate}[wide = 0pt, leftmargin = *]
\item[] $\ml{R}_1$: If $A\rightarrowast B\leftcircleast R$, and $A$ and $R$ are not adjacent, then orient the triple as $A\rightarrowast B\rightarrow R$.
 \item[] $\ml{R}_2$: If $A\rightarrow B\rightarrowast R$ or $A\rightarrowast B\rightarrow R$, and $A\rightcircleast R$, then orient $A\rightcircleast R$ as $A\rightarrowast R$.
 \item[] $\ml{R}_3$: If $A\rightarrowast B\leftarrowast R$, $A\rightcircleast D\leftcircleast R$, $A$ and $R$ are not adjacent, and $D\rightcircleast B$, then orient $D\rightcircleast B$ as $D\rightarrowast B$.
 \item[] $\ml{R}_4$: If $\langle K,\dots,A,B,R\rangle$ is a discriminating path between $K$ and $R$ for $B$, and $B\leftcircleast R$; then if $B\in \text{Sepset}(K,R)$, orient $B\leftcircleast R$ as $B\rightarrow R$; otherwise orient the triple $\langle A,B,R\rangle$ as $A\leftrightarrow B\leftrightarrow R$.
 \item[] $\ml{R}_8$: If $A\rightarrow B\rightarrow R$, and $A\rightarrowcircle R$, orient $A\rightarrowcircle R$ as $A\rightarrow R$.
 \item[] $\ml{R}_9$: If $A\rightarrowcircle R$, and $p=\langle A,B,D,\dots,R\rangle$ is an uncovered possible directed path from $A$ to $R$ such that $R$ and $B$ are not adjacent, then orient $A\rightarrowcircle R$ as $A\rightarrow R$.
 \item[] $\ml{R}_{10}$: Suppose $A\rightarrowcircle R$, $B\rightarrow R\leftarrow D$, $p_1$ is an uncovered possible directed path from $A$ to $B$, and $p_2$ is an uncovered possible directed path from $A$ to $D$. Let $U$ be the vertex adjacent to $A$ on $p_1$ ($U$ could be $B$), and $W$ be the vertex adjacent to $A$ on $p_2$ ($W$ could be $D$). If $U$ and $W$ are distinct, and are not adjacent, then orient $A\rightarrowcircle R$ as $A\rightarrow R$.
 \item[] $\ml{R}_4'$: If $\langle K,\cdots,A,B,R\rangle$ is a discriminating path between $K$ and $R$ for $B$, and $B\leftcircleast R$, then orient $B\leftcircleast R$ as $B\rightarrow R$.
 \item[] $\ml{R}_{11}$: If $A\rightcircle B$, then $A\rightarrow B$.
\end{enumerate}

\subsection{Preliminary about Causal Effect Estimation}
\label{subsec:preliminary about causal effect}
\begin{myDef}[Adjustment set;~\citet{books/2009causality,conf/uai/ZanderLT14a}]
\label{def:adjustment set}
Given a DAG, MAG, or PAG $G$, $\mf{Z}$ is an adjustment set relative to $(X,Y)$ if for any probability density $f$ compatible with $G$, the causal effect of $X$ on $Y$
\begin{align}
\label{eq:adjustment criterion}
P(Y|do(X))=\left\{\begin{array}{ll}
P(Y|X), & \text { if } \mathbf{Z}=\varnothing, \\
\int_{\mathbf{Z}} P(Y|\mathbf{Z}, X) P(\mathbf{Z}) \diff \mathbf{Z}, & \text { otherwise. }
\end{array}\right.
\end{align}
\end{myDef}
\citet{journals/AOS/maathuis2015,journals/jmlr/PerkovicTKM17} presented the necessary and sufficient graphical characterization for the causal effect identifiablility via covariate adjustment given a DAG/CPDAG/MAG/PAG. We show them in Prop.~\ref{prop: improved generalized backdoor}. See \citet{journals/AOS/maathuis2015} for the notation $G_{\underline{X}}$. At first, we introduce an important concept ${\normalfont \mbox{D-SEP}(X,Y,G)}$ in Def.~\ref{def:d-sep}. ${\normalfont \mbox{D-SEP}(X,Y,G)}$ is essentially a set of vertices.
\begin{myDef}[${\normalfont \mbox{D-SEP}(X,Y,G)}$;~\citet{books/spirtes2000causation}]
\label{def:d-sep}
Let $X$ and $Y$ be two distinct vertices in a mixed graph $G$. We say that $V\in {\normalfont \mbox{D-SEP}(X,Y,G)}$ if $V\not=X$, and there is a collider path between $X$ and $V$ in $G$, such that every vertex on this path (including $V$) is an ancestor of $X$ or $Y$ in $G$.
\end{myDef}
\begin{prop}[\citet{journals/AOS/maathuis2015,journals/jmlr/PerkovicTKM17}]
\label{prop: improved generalized backdoor}
Let $G$ be a MAG or PAG, and $X$ and $Y$ be two distinct vertices in $G$. There exists an adjustment set relative to $(X,Y)$ in $G$ if and only if $Y\not\in {\normalfont \mbox{Adj}}(X,G_{\underline{X}})$ and ${\normalfont \mbox{D-SEP}(X,Y,G_{\underline{X}})}\cap {\normalfont \mbox{PossDe}}(X,G)=\emptyset$. Moreover, if an adjustment set exists, then ${\normalfont \mbox{D-SEP}(X,Y,G_{\underline{X}})}$ is such a set. Denote ${\normalfont \mbox{D-SEP}(X,Y,G_{\underline{X}})}$ by $\mf{D}$, then
\begin{align}
\label{eq:gbc}
f(Y|do(X=x))=\int_{\mf{D}} f(\mf{D})f(Y|\mf{D},X=x)\diff \mf{D}.
\end{align}
\end{prop}

\section{A Detailed Introduction to PAGcauses (Wang et al. 2023a)}
\label{sec:More Related Results}
In this part, we introduce the method PAGcaused to determine the set of possible causal effects given a PAG, which is proposed by~\citet{conf/icml/WangQZ23}. We first show some theoretical results of~\citet{conf/icml/WangQZ23}. Some of these results are needed in our proof.

The first important result is Prop.~\ref{thm:dag and mag}. It provides a graphical characterization for the adjustment set comprised of observable variables in all the DAGs represented by a given MAG. It implies that there exists a DAG where the causal effect of $X$ on $Y$ is identifiable by covariate adjustment and the adjustment set is comprised of some observable variables if and only if ${\normalfont \mbox{D-SEP}}(X,Y,\ml{M}_{\utilde{X}})\cap {\normalfont \mbox{De}(X,\ml{M})}=\emptyset$, and ${\normalfont \mbox{D-SEP}}(X,Y,\ml{M}_{\utilde{X}})$ is the adjustment set. Hence, given a MAG $\ml{M}$, we can determine whether there exists some DAG where the causal effect is identifiable by adjusting for observable variables without the need the enumerate the DAGs. And for all the DAGs above, they are associated with the same causal effect.
\begin{prop}
\label{thm:dag and mag}
Suppose a MAG $\ml{M}$ where $X\in {\normalfont \mbox{Anc}(Y,\ml{M})}$. There exists a DAG $\ml{D}$ represented by $\ml{M}$ such that the causal effect of $X$ on $Y$ in $\ml{D}$ can be identified by adjusting for a set comprised of $\mf{V}(\ml{M})$ if and only if ${\normalfont \mbox{D-SEP}}(X,Y,\ml{M}_{\utilde{X}})\cap {\normalfont \mbox{De}(X,\ml{M})}=\emptyset$. Furthermore, if such a set exists, ${\normalfont \mbox{D-SEP}}(X,Y,\ml{M}_{\utilde{X}})$ is an adjustment set. Hence, in the following, when we say an adjustment set in $\ml{M}$, it means the adjustment set in the DAGs represented by the MAG $\ml{M}$ such that the causal effect of $X$ on $Y$ can be identified by adjusting for this adjustment set.
\end{prop}
According to Prop.~\ref{thm:dag and mag}, when addressing the set determination task for the causal effect of $X$ on $Y$, we only need to consider the MAG in the MEC presented by $\ml{P}$, without the need to consider the DAGs. Considering there are many circles in a PAG $\ml{P}$, PAGcauses considers transforming the circles of $X$ at first. They use a set of vertices $\mf{C}$ to represent a local transformation of $X$, \emph{i.e.}, transform $X\leftcircleast V$ to $X\leftarrowast V$ if $V\in \mf{C}$ and transform $X\leftcircleast V$ to $X\rightast V$ if $V\in \{V\in \mf{V}(\ml{P})|X\leftcircleast V\mbox{ in }\ml{P}\}\backslash \mf{C}$. An important problem here is to evaluate the validity of each local transformation, \emph{i.e.}, whether there is a MAG consistent with this local transformation in the MEC represented by $\ml{P}$. For this purpose, they presented a graphical characterization for the valid local transformation of $X$ implied by $\mf{C}$ given the PAG $\ml{P}$, which is shown in Prop.~\ref{thm:existence}. The three conditions in Prop.~\ref{thm:existence} can be evaluated in $\ml{O}(d^3)$, where $d$ denotes the number of vertices. In Prop.~\ref{thm:existence}, the concept of bridged is involved, which is shown in Def.~\ref{def:bridged}.
\begin{myDef}[Bridged relative to $\mf{V}'$ in $H$;~\citeauthor{conf/icml/WangQZ23}~\citeyear{conf/icml/WangQZ23}]
\label{def:bridged}
Let $H$ be a partial mixed graph. Denote $H'$ a subgraph of $H$ induced by a set of vertices $\mf{K}$. Given a set of vertices $\mf{V}'$ in $H$ that is disjoint of $\mf{K}$, two vertices $A$ and $B$ in a connected circle component of $H'$ are \emph{bridged relative to $\mf{V}'$} if either $A=B$ or in each minimal circle path $A(=V_0)\leftrightcircle V_1\leftrightcircle\cdots\leftrightcircle V_n\leftrightcircle B(=V_{n+1})$ from $A$ to $B$ in $H'$, there exists one vertex $V_s,0\leq s\leq n+1$, such that $\ml{F}_{V_i}\subseteq \ml{F}_{V_{i+1}},0\leq i\leq s-1$ and $\ml{F}_{V_{i+1}}\subseteq \ml{F}_{V_i},s\leq i\leq n$, where $\ml{F}_{V_i}=\{V\in \mf{V}'\mid V\rightcircleast V_i \mbox{ or } V\rightarrowast V_i\mbox{ in }H\}$. Further, $H'$ is \emph{bridged relative to $\mf{V}'$ in $H$} if any two vertices in a connected circle component of $H'$ are bridged relative to $\mf{V}'$.
\end{myDef}
\begin{prop}
\label{thm:existence}
Given a PAG $\ml{P}$, for any set of vertices $\mf{C}\subseteq \{V\mid X \leftcircleast V \mbox{ in }\ml{P}\}$, there exists a MAG $\ml{M}$ consistent with $\ml{P}$ with $X\leftarrowast V$, $\forall V\in \mf{C}$ and $X\rightarrow V$, $\forall V\in \{V\mid X\leftcircleast V \mbox{ in }\ml{P}\}\backslash \mf{C}$ if and only if
\begin{enumerate}
\item[(1)] ${\normalfont \mbox{PossDe}}(X,\ml{P}[-\mf{C}])\cap {\normalfont \mbox{Pa}}(\mf{C},\ml{P})=\emptyset$;
\item[(2)] the subgraph $\ml{P}[\mf{C}]$ of $\ml{P}$ induced by $\mf{C}$ is a complete graph;
\item[(3)] $\ml{P}[{\normalfont \mbox{PossDe}}(X,\ml{P}[-\mf{C}])\backslash \{X\}]$ is bridged relative to $\mf{C}\cup \{X\}$ in $\ml{P}$.
\end{enumerate}
\end{prop}
After incorporating a valid local transformation into the PAG $\ml{P}$,~\citet{conf/icml/WangQZ23} used the sound and complete orientation rules $\ml{R}_1-\ml{R}_3,\ml{R}_4',\ml{R}_7-\ml{R}_{10},\ml{R}_{11}$ to further update the graph. And the obtained graph is called by maximal local MAG, denoted by $\mb{M}$. By incorporating different valid local transformations, PAGcauses can obtain different maximal local MAGs. However, determining a maximal local MAG $\mb{M}$ is not sufficient for determining the only adjustment set in all the MAGs consistent with $\mb{M}$. Hence, they established the graphical characterization for the adjustment set in the MAGs consistent with $\mb{M}$ in Prop.~\ref{prop:existencecritical} and Prop.~\ref{thm:existenceDSEP}. Both Prop.~\ref{prop:existencecritical} and Prop.~\ref{thm:existenceDSEP} ensure that PAGcauses can find the same set of causal effects given $\mb{M}$ as methods to enumerate all the MAGs consistent with $\mb{M}$.
\begin{restatable}{prop}{thmDSEP}
\label{thm:existenceDSEP}
Given a maximal local MAG $\mb{M}$, suppose a MAG $\ml{M}$ consistent with $\mb{M}$ such that there exists an adjustment set relative to $(X,Y)$. Let $\mf{W}$ be ${\normalfont \mbox{D-SEP}}(X,Y,\ml{M}_{\utilde{X}})$. Then $\mf{W}$ is a potential adjustment set in $\mb{M}$ and there exists a block set $\mf{S}$ such that
\begin{enumerate}
\item[(1)] ${\normalfont \mbox{PossDe}}(\bar{\mf{W}},\mb{M}[-\mf{S}])\cap {\normalfont \mbox{Pa}}(\mf{S},\mb{M})=\emptyset$;      
\item[(2)] $\mb{M}[\mf{S}_{V}]$ is a complete graph for any $V\in \bar{\mf{W}}$, where $\mf{S}_{V}=\{V'\in\mf{S}|V\leftcircleast V'\mbox{ in }\mb{M}\}$; 
\item[(3)] $\mb{M}[{\normalfont \mbox{PossDe}}(\bar{\mf{W}},\mb{M}[-\mf{S}])]$ is bridged relative to $\mf{S}$ in $\mb{M}$.
\end{enumerate}
\end{restatable}
With the results above, they proposed their method PAGcauses in Alg.~\ref{alg:obtain causal bounds trivially}.
\begin{algorithm}[htbp]
    \SetAlgoLined
    \KwIn{PAG $\ml{P}$, $X$, $Y$}
    $\widehat{\mbox{\normalfont AS}}(\ml{P})= \emptyset$ \Comment{Record all the valid adjustment sets}\;
    \lIf{$X\not\in{\normalfont \mbox{PossAn}}(Y,\ml{P})$}{\Return No causal effects}
    \uIf{the conditions in Prop.~\ref{prop: improved generalized backdoor} are satisfied for $\ml{P}$}{\Return $\widehat{\mbox{\normalfont AS}}(\ml{P})\leftarrow \{{\normalfont \mbox{D-SEP}}(X,Y,\ml{P}_{\underline{X}})\}$\Comment{Prop.~\ref{prop: improved generalized backdoor}}}
    \For{each set $\mf{C}\subseteq \{V\mid V\rightcircleast X\mbox{ in }\ml{P}\}$}{
        \uIf{the three conditions in Prop.~\ref{thm:existence} are satisfied}{
        Obtain a maximal local MAG $\mb{M}$ based on $\ml{P}$ and $\mf{C}$\;
        Find all potential adjustment sets $\mf{W}_1,\mf{W}_2,\cdots$ given $\mb{M}$ according to Def.~\ref{def: possible adjust}\;
        \For{each potential adjustment set $\mf{W}_i$}{
            \For{each block set $\mf{S}$}{
                \uIf{the three conditions in Prop.~\ref{prop:existencecritical} are satisfied given $\mf{S}$}{
                    $\widehat{\mbox{\normalfont AS}}(\ml{P})\leftarrow \widehat{\mbox{\normalfont AS}}(\ml{P})\cup \{\mf{W}_i\}$\;
                    \tb{break}\Comment{Break the loop of $\mf{S}$}\;
                }
            }
        }}
    }
    \KwOut{Set of causal effects via adjustment in the given PAG $\ml{P}$ identified with $\widehat{\mbox{\normalfont AS}}(\ml{P})$ by~\eqref{eq:adjustment criterion}}
    \caption{PAGcauses}
    \label{alg:obtain causal bounds trivially}
\end{algorithm}

In the proof of Prop.~\ref{prop:existencecritical}, if the three conditions are satisfied, they present an algorithm to construct a MAG consistent with $\mb{M}$ such that the adjustment set is $\mf{W}$. We show the algorithm in Alg.~\ref{algo:maximal localMAG to MAG}.
\renewcommand{\algorithmicrequire}{\textbf{input:}}
\renewcommand{\algorithmicensure}{\textbf{output:}}
\begin{algorithm}[tb]
     \caption{Orient a maximal local MAG of $X$ as a MAG}
    \label{algo:maximal localMAG to MAG}
    \begin{algorithmic}[1]
    \Require Maximal local MAG $\mb{M}$, potential adjustment set $\mf{W}$ and corresponding $\mf{\bar{W}}$ according to Def.~\ref{def:T1 and T2}, block set $\mf{S}$
    \State for $\forall K\in \text{PossDe}(\bar{\mf{W}},\mb{M}[-\mf{S}])$ and $\forall T\in \mf{S}$ such that $K\leftcircleast T$ in $\mb{M}$, orient it as $K\leftarrowast T$ (the mark at $T$ remains);
    \State update the subgraph $\mb{M}[\text{PossDe}(\bar{\mf{W}},\mb{M}[-\mf{S}])]$ as follows until no feasible updates: for any two vertices $V_i$ and $V_{j}$ such that $V_i\leftrightcircle V_{j}$, orient it as $V_i\rightarrow V_{j}$ if (1) $\ml{F}_{V_i}\backslash \ml{F}_{V_{j}}\not=\emptyset$ or (2) $\ml{F}_{V_i}=\ml{F}_{V_{j}}$ as well as there is a vertex $V_k\in \text{PossDe}(\bar{\mf{W}},\mb{M}[-\mf{S}])$ not adjacent to $V_j$ such that $V_k\rightarrow V_i\leftrightcircle V_j$, where $\ml{F}_{V_i}=\{V\in \mf{S}\mid V\rightcircleast V_i\mbox{ in }\mb{M}\}$; 
    \State orient the circles on the remaining $\rightarrowcircle$ edges as tails;
    \State in subgraph $\mb{M}[\text{PossDe}(\bar{\mf{W}},\mb{M}[-\mf{S}])]$, orient the circle component into a DAG without new unshielded colliders;
   \State in subgraph $\mb{M}[-\text{PossDe}(\bar{\mf{W}},\mb{M}[-\mf{S}])]$, orient the circle component into a DAG without new unshielded colliders.
    \Ensure A MAG $\ml{M}$
\end{algorithmic}
\end{algorithm}%

\section{Complexity Analysis of Algorithm~\ref{alg: R1213}}
\label{sec:complexity of rules}
Suppose there are $m$ edges and $d$ vertices in $H$. There are $m$ edges that can be transformed. Hence the round of loop is $\ml{O}(m)$. In each round, suppose we want to detect whether $A\leftcircleast B$ can be transformed by $\ml{R}_{12}$ or $\ml{R}_{13}$. The complexity of Line 2 and Line 3 is $\ml{O}(d)$. And determining the set $\text{Anc}(\mf{S}_A,H)$ takes a $\ml{O}(md)$ complexity. Executing Line 7 and Line 9 take a $\ml{O}(m)$ complexity. And the complexity of Line 8 is $\ml{O}(m^2)$. Hence the complexity of Alg.~\ref{alg: R1213} is $\ml{O}(m^3)$.
\section{Proof}
\label{sec: proofs}
\subsection{Proof of Proposition~\ref{prop:implementation soundess}}
\label{sec:prop implementation of R12R13}
\begin{proof}
For $\ml{R}_{13}$, the implementation on Line 5 of Alg.~\ref{alg: R1213} as well as the conditions on Line 2 - Line 4 totally follow the conditions in $\ml{R}_{13}$. Hence it evidently that for the edges that can be transformed by $\ml{R}_{13}$, Alg.~\ref{alg: R1213} can transform these edges, and will not transform other edges on Line 5. It suffices to show that in addition to the edges transformed to $\ml{R}_{13}$, Alg.~\ref{alg: R1213} can transform and only transform the edges triggered by $\ml{R}_{12}$ in the following.

We first prove that if $\ml{R}_{12}$ is triggered in $H$, then Line 9 of Alg.~\ref{alg: R1213} will transform this edge. Suppose there is an unbridged path $V_0\leftrightcircle V_1\leftrightcircle\cdots \leftrightcircle V_n$ relative to $\mf{S}_A$ in $H$. Then there is $\ml{F}_{V_0}\backslash \ml{F}_{V_1}\neq \emptyset$ and $\ml{F}_{V_n}\backslash \ml{F}_{V_{n-1}}\neq \emptyset$. Suppose $S_1\in \ml{F}_{V_0}\backslash \ml{F}_{V_1}$ and $S_2\in \ml{F}_{V_n}\backslash \ml{F}_{V_{n-1}}$. According to Line 7 of Alg.~\ref{alg: R1213}, there is $S_1\rightarrowast V_0$ and $V_n\leftarrowast S_2$ in $H'$. Note $S_1$ cannot be adjacent to $V_1$, for otherwise there can only be $V_1\rightarrow S_1$ in $H$, for otherwise there is $S_1\in \ml{F}_{V_1}$. However, in this case Line 5 of Alg.~\ref{alg: R1213} has been triggered, thus Line 7 and 8 will not be executed. Similarly, $S_2$ is not adjacent to $V_{n-1}$. Hence there is an uncovered path $S_1\rightarrowast V_0\leftrightcircle V_1\leftrightcircle\cdots \leftrightcircle V_n\leftarrowast S_2$ in $H'$. According to the update of Line 8 of Alg.~\ref{alg: R1213}, there must be an unshielded collider in this path after the update. Thus Line 9 will be triggered and $A\leftcircleast B$ can be transformed by $A\leftarrowast B$.

Then we prove that if Alg.~\ref{alg: R1213} transforms an edge $A\leftcircleast B$ to $A\leftarrowast B$ on Line 9, then $\ml{R}_{12}$ can be triggered. Since the rigorous proof is somewhat tedious, we just show a proof sketch here. If Alg.~\ref{alg: R1213} transforms an edge $A\leftcircleast B$ to $A\leftarrowast B$ on Line 9, then there is an unshielded collider formed in $H'$ on Line 8. In this case, according to the proof idea of Lemma 5 of~\citet{conf/icml/WangQZ23}, we can prove that $H[\mf{D}]$ is not bridged relative to $\mf{S}_A$. Then, we will prove that there is an unbridged path in $H[\mf{D}]$ relative to $\mf{S}_A$.

Suppose there is a minimal path $p=V_0\leftrightcircle \cdots \leftrightcircle V_{n+1}$ in $H[\mf{D}]$ where there is not a vertex $V_s$ such that $\ml{F}_{V_i}\subseteq \ml{F}_{V_{i+1}},0\leq i\leq s-1$ and $\ml{F}_{V_{i+1}}\subseteq \ml{F}_{V_i},s\leq i\leq n$. As Line 9 is triggered, no vertex in $\mf{D}$ belongs to $\text{Anc}(\mf{S}_A,H)$. Hence if there is some vertex $V\not\in \ml{F}_{V_i}$, $V$ is not adjacent to $V_i$. Next, we consider the path $p$. We discuss whether $\ml{F}_{V_0}\backslash \ml{F}_{V_1}=\emptyset$. If empty, we consider whether $\ml{F}_{V_1}\backslash \ml{F}_{V_2}=\emptyset$ instead. We repeat the process above until we find the first index $j$ such that $\ml{F}_{V_j}\backslash \ml{F}_{V_{j+1}}\not=\emptyset$. Note such $j$ must exist, for otherwise, there is $\ml{F}_{V_0}\subseteq \ml{F}_{V_1}\subseteq \ml{F}_{V_2}\subseteq\cdots \subseteq\ml{F}_{V_{n+1}}$, in which case there is $s={n+1}$ such that $\ml{F}_{V_i}\subseteq \ml{F}_{V_{i+1}},0\leq i\leq s-1$ and $\ml{F}_{V_{i+1}}\subseteq \ml{F}_{V_i},s\leq i\leq n$, contradiction. Then, we consider the sub-path $V_j\leftrightcircle \cdots \leftrightcircle V_{n+1}$. Note according to the process above, there is $\ml{F}_{V_0}\subseteq \ml{F}_{V_1}\subseteq \ml{F}_{V_2}\subseteq\cdots \subseteq  \ml{F}_{V_j}$ and $\ml{F}_{V_j}\backslash \ml{F}_{V_{j+1}}\not=\emptyset$. Then, we consider whether $\ml{F}_{V_{n+1}}\backslash \ml{F}_{V_n}=\emptyset$. If empty, we consider whether $\ml{F}_{V_n}\backslash \ml{F}_{V_{n-1}}=\emptyset$ instead. We repeat the process above until we find the first index $k$ such that $\ml{F}_{V_k}\backslash \ml{F}_{V_{k-1}}\not=\emptyset$. Similar to the proof above, such index $k$ must exist. And there is $\ml{F}_{V_k}\supseteq \ml{F}_{V_{k+1}}\supseteq \cdots \supseteq  \ml{F}_{V_{n+1}}$ and $\ml{F}_{V_{k-1}}\not=\emptyset$. Next we discuss the relationship between $j$ and $k$. We will prove the impossibility of $k\leq j$. Suppose $k\leq j$. According to the result above, there is $\ml{F}_{V_0}\subseteq \ml{F}_{V_1}\subseteq \ml{F}_{V_2}\subseteq\cdots \subseteq  \ml{F}_{V_j}$ and $\ml{F}_{V_k}\supseteq \ml{F}_{V_{k+1}}\supseteq \cdots \supseteq  \ml{F}_{V_{n+1}}$. Hence, there is $\ml{F}_{V_0}\subseteq \ml{F}_{V_1}\subseteq \ml{F}_{V_2}\subseteq\cdots \subseteq \ml{F}_{V_k}=\ml{F}_{V_{k+1}}=\cdots=\ml{F}_{V_j}\supseteq\cdots \supseteq\ml{F}_{V_{n+1}}$. In this case, there is $s={j}$ such that $\ml{F}_{V_i}\subseteq \ml{F}_{V_{i+1}},0\leq i\leq s-1$ and $\ml{F}_{V_{i+1}}\subseteq \ml{F}_{V_i},s\leq i\leq n$, contradiction. Hence there is $k\geq j+1$. In this case, there is a minimal path $V_j\leftrightcircle V_{j+1}\leftrightcircle\cdots\leftrightcircle V_k$ such that $\ml{F}_{V_j}\backslash \ml{F}_{V_{j+1}}\not=\emptyset$ and $\ml{F}_{V_{k}}\backslash\ml{F}_{V_{k-1}}\not=\emptyset$ in $H[\mf{D}]$. It is an unbridged path relative to $\mf{S}_A$ in $H[\mf{D}]$.

And since there are an uncovered possible directed paths from $A$ to any vertice in $\mf{D}$ and at the same time $B$ is the vertex adjacent to $A$ in the paths, $\ml{R}_{12}$ can be triggered to transform $A\leftcircleast B$ to $A\leftarrowast B$. The proof completes.
\end{proof}
\subsection{Proof of Theorem~\ref{Thm:rule 12 all vertices}}
\label{sec:Proof of Thm:rule 12}
\begin{proof}
We first prove the soundness of $\ml{R}_{13}$, then prove the soundness of $\ml{R}_{12}$.

For $\ml{R}_{13}$, suppose there is an MAG $\ml{M}$ with $A\rightarrow B$. It is evident that the uncovered path is $A\rightarrow B\rightarrow \cdots \rightarrow K$. According to the definition of $\mf{S}_A$ and the conditions in $\ml{R}_{13}$, there must be a vertex $C\in \mf{S}_A$ such that there is $C\rightarrowast A\rightarrow \cdots K\rightarrow \cdots C$, which contradicts with the ancestral property, contradiction.

For $\ml{R}_{12}$, suppose there is an MAG $\ml{M}$ with $A\rightarrow B$. According to the condition of $\ml{R}_{12}$, it is evident that for each vertex $K_i,1\leq i\leq m$, there is a minimal directed path $A\rightarrow B\rightarrow \cdots\rightarrow K_i$ in $\ml{M}$. And for any $T\in \mf{S}_A$ and $K_i,1\leq i\leq m$ such that there is an edge between $T$ and $K_i$ in the PMG $H$, there must be $T\rightarrowast K_i$ in $\ml{M}$, for otherwise there is $A\rightarrow B\rightarrow \cdots \rightarrow K_i\rightarrow T\rightarrowast A$, contradicting with ancestral property. According to Def.~\ref{def:unbridged path}, there is an unbridged path $p:K_1\leftrightcircle K_2\leftrightcircle \cdots\leftrightcircle K_m$ such that $\ml{F}_{K_1}\backslash \ml{F}_{K_2}\neq\emptyset$ and $\ml{F}_{K_m}\backslash \ml{F}_{K_{m-1}}\neq\emptyset$. 

Suppose $C_1\in \ml{F}_{K_1}\backslash \ml{F}_{K_2}$ and $C_2\in \ml{F}_{K_m}\backslash \ml{F}_{K_{m-1}}$. If $C_1$ is adjacent to $K_2$, as $C_1\not\in \ml{F}_{K_2}$, there is $K_2\rightarrow C_1$ in $H$. In this case there must be $A\leftarrowast B$ according to $\ml{R}_{13}$, the soundness of which has been proven. Thus $A\rightarrow B$ and that $C_1$ is adjacent to $K_2$ are impossible. In the following, we consider the case that $C_1$ is not adjacent to $K_2$ and $C_2$ is not adjacent to $K_{m-1}$. We have shown before that if there is an edge $A\rightarrow B$ in $\ml{M}$, there must be $C_1\rightarrowast K_1$ and $C_2\rightarrowast K_m$ in $\ml{M}$. In this case, no matter how we transform the circles in $C_1\rightarrowast K_1\leftrightcircle K_2\leftrightcircle \cdots\leftrightcircle K_m\leftarrowast C_2$, there will be a new unshielded collider in $\ml{M}$, which contradicts with the fact that $\ml{M}$ is consistent with $H$.
\end{proof}
\subsection{Proof of Theorem~\ref{prop:existencecritical rules}}
\begin{proof}[Proof of Theorem~\ref{prop:existencecritical rules}]
Note the graph $\mb{M}$ is updated in every round of Alg.~\ref{alg:get S}. To distinguish them, we use $\mb{M}_i$ to denote the graph obtained after Line 5 of Alg.~\ref{alg:get S} in the $i$-th round. Denote $\mb{M}$ the original maximal local MAG. Note in the whole process, there are no new tails introduced. Hence, for any $i$, ${\normalfont \mbox{Anc}}(Y,\mb{M}_i)={\normalfont \mbox{Anc}}(Y,\mb{M})$. For brevity, denote $\mf{T}={\normalfont \mbox{PossDe}}(\bar{\mf{W}},\mb{M}[-\mf{S}])\backslash \bar{\mf{W}})$ like Line 2 of Alg.~\ref{alg:get S}. Suppose there are $T$ rounds in Alg.~\ref{alg:get S}, where in the $i$-th round, $1\leq i\leq T-1$, an edge $A_i\leftarrowast S_i$ is transformed by $\ml{R}_{12}$ on Line 6 of Alg.~\ref{alg:get S} ($A_i$ just denotes any a vertex), and thus $\text{Anc}(S_i,\mb{M}_i)\cap \mf{T}$ is incorporated to $\mf{S}$ on Line 7 of Alg.~\ref{alg:get S}.
 
Since the algorithm returns a set of vertices $\mf{S}$ after $T$ rounds, according to Line 4, 6, and 9 of Alg.~\ref{alg:get S}, we conclude that (1) $\mb{M}_T[\mf{S}_V]$ is a complete graph for any $V\in \bar{\mf{W}}$, where $\mf{S}_V=\{V'\in \mf{S}|V\leftcircleast V'\mbox{or }V\leftarrowast V'\mbox{ in }\mb{M}_T\}$, (2) ${\normalfont \mbox{PossDe}}(\bar{\mf{W}},\mb{M}_T[-\mf{S}])\cap {\normalfont \mbox{Pa}}(\mf{S},\mb{M}_T)=\emptyset$, (3) there is no unbridged path relative to $\mf{S}$ in $\mb{M}_T[\text{PossDe}(\bar{\mf{W}},\mb{M}_T[-\mf{S}])]$. We will prove that the three conditions in Prop.~\ref{prop:existencecritical} are satisfied given the set $\mf{S}$, thus we can conclude the desired result by Prop.~\ref{prop:existencecritical}.

(1) We will prove that $\mb{M}[\mf{S}_V]$ is a complete graph for any $V\in \bar{\mf{W}}$, where $\mf{S}_V=\{V'\in \mf{S}|V\leftcircleast V'\mbox{or }V\leftarrowast V'\mbox{ in }\mb{M}\}$. In the process of Alg.~\ref{alg:get S}, we only transform some circles at $V\in \bar{\mf{W}}$ to arrowheads, hence $\mb{M}[\mf{S}_V]=\mb{M}_T[\mf{S}_V]$ for any $V\in \bar{\mf{W}}$. Hence given $\mb{M}_T[\mf{S}_V]$ is a complete graph for any $V\in \bar{\mf{W}}$, $\mb{M}[\mf{S}_V]$ is a complete graph for any $V\in \bar{\mf{W}}$.

(2) We will prove that ${\normalfont \mbox{PossDe}}(\bar{\mf{W}},\mb{M}[-\mf{S}])\cap {\normalfont \mbox{Pa}}(\mf{S},\mb{M})=\emptyset$. Suppose $A= {\normalfont \mbox{PossDe}}(\bar{\mf{W}},\mb{M}[-\mf{S}])\cap {\normalfont \mbox{Pa}}(\mf{S},\mb{M})$. Without loss of generality, suppose there is $V\in \bar{\mf{W}}$ and a minimal possible directed path $p=\langle V,J_1,\cdots,J_k,A\rangle$ from $V$ to $A$ in $\mb{M}$ such that each non-endpoint in $p$ does not belong to $\bar{\mf{W}}$ (If there is $V'\in \bar{\mf{W}}$ in the path $p$, we consider $V'$ instead of $V$ and the sub-path from $V'$ to $A$ instead of $p$). Note in the process of Alg.~\ref{alg:get S}, we only transform some circles at $V\in \bar{\mf{W}}$ to arrowheads on the edges connecting $\bar{\mf{W}}$ and $\mf{S}$. Hence, since $p$ is not a minimal possible directed path in $\mb{M}_T$, and each non-endpoint does not belong to $\bar{\mf{W}}$, there must be $J_1\in \mf{S}$. However, in this case the path $p$ is not in $\mb{M}[-\mf{S}]$ since $J_1\in \mf{S}$, contradiction. Hence ${\normalfont \mbox{PossDe}}(\bar{\mf{W}},\mb{M}[-\mf{S}])\cap {\normalfont \mbox{Pa}}(\mf{S},\mb{M})=\emptyset$.

(3) We will prove that there is no unbridged path relative to $\mf{S}$ in $\mb{M}[\text{PossDe}(\bar{\mf{W}},\mb{M}[-\mf{S}])]$. Note in the process of Alg.~\ref{alg:get S}, we only transform some circles at $V\in \bar{\mf{W}}$ to arrowheads on the edges connecting $\bar{\mf{W}}$ and $\mf{S}$. Hence, if there is an unbridged path relative to $\mf{S}$ in $\mb{M}$, this path still exists in $\mb{M}_T$ since each vertex in this path cannot belong to $\mf{S}$, which concludes that $\mb{M}_T[{\normalfont \mbox{PossDe}}(\bar{\mf{W}},\mb{M}_T[-\mf{S}])]$ is not bridged relative to $\mf{S}$ in $\mb{M}_T$, contradiction. Hence there is no unbridged path relative to $\mf{S}$ in $\mb{M}[\text{PossDe}(\bar{\mf{W}},\mb{M}[-\mf{S}])]$, that is, $\mb{M}[{\normalfont \mbox{PossDe}}(\bar{\mf{W}},\mb{M}[-\mf{S}])]$ is bridged relative to $\mf{S}$ in $\mb{M}$ according to Def.~\ref{def:bridged}.
\end{proof}
\subsection{Proof of Theorem~\ref{thm:existenceDSEP rules}}
The proof relied on some results by~\citet{conf/icml/WangQZ23}. We first present two supporting results in Lemma~\ref{lemma:M P same} and Lemma~\ref{lemma:mpdp}. Lemma~\ref{lemma:mpdp} implies that if there is a possible directed path from $A$ to $B$ in a maximal local MAG $\mb{M}$, then we can find a minimal possible directed path from $A$ to $B$ in $\mb{M}$.
\begin{lemma}[\citet{conf/icml/WangQZ23}]
\label{lemma:M P same}
Given a maximal local MAG $\mb{M}$ obtained from a PAG $\ml{P}$ and a valid local transformation of $X$ represented by $\mf{C}$, the following properties are satisfied:
\begin{enumerate}
\item[]\textbf{(Invariant)} The arrowheads and tails in $\mb{M}$ are invariant in all the MAGs consistent with $\ml{P}$ and the local transformation of $X$ represented by $\mf{C}$;
\item[]\textbf{(Chordal)} the circle component in $\mb{M}$ is chordal;
\item[]\textbf{(Balanced)} for any three vertices $A,B,C$ in $\mb{M}$, if $A\rightarrowast B \leftcircleast C$, then there is an edge between $A$ and $C$ with an arrowhead at $C$, namely, $A\rightarrowast C$. Furthermore, if the edge between $A$ and $B$ is $A\rightarrow B$, then the edge between $A$ and $C$ is either $A\rightarrow C$ or $A\rightarrowcircle C$ (i.e., it is not $A\leftrightarrow C)$;
\item[]\textbf{(Complete)} for each circle at vertex $A$ on any edge $A\leftcircleast B$ in $\mb{M}$, there exist MAGs $\ml{M}_1$ and $\ml{M}_2$ consistent with $\mb{M}$ with $A\leftarrowast B\in \mf{E}(\ml{M}_1)$ and $A\rightarrow B\in \mf{E}(\ml{M}_2)$;
\item[]\textbf{(P6)} we can always obtain a MAG consistent with $\ml{P}$ and the local transformation of $X$ represented by $\mf{C}$, by transforming the circle component into a DAG without unshielded colliders and transforming $A\rightarrowcircle B$ as $A\rightarrow B$.
\end{enumerate}
\end{lemma}
\begin{lemma}[\citet{conf/icml/WangQZ23}]
\label{lemma:mpdp}
Consider a maximal local MAG $\mb{M}$. If there is a possible directed path from $A$ to $B$ in $\mb{M}$, then there is a minimal possible directed path from $A$ to $B$ in $\mb{M}$.
\end{lemma}
\begin{lemma}
\label{lemma: possDe of barW}
Given a maximal local MAG $\mb{M}$, suppose a MAG $\ml{M}$ consistent with $\mb{M}$ such that there exists an adjustment set relative to $(X,Y)$. Let $\mf{W}$ be ${\normalfont \mbox{D-SEP}}(X,Y,\ml{M}_{\utilde{X}})$. Suppose there is a minimal possible directed path $p=\langle J_0(=V),J_1,\cdots,J_s(=T)$ from $V\in \bar{\mf{W}}$ to a vertex $T$ in $\mb{M}$, where each non-endpoint in $p$ does not belong to $\mf{W}\cup\bar{\mf{W}}$. If $T\in {\normalfont \mbox{Anc}}(Y,\ml{M})$, then $p$ can only be as $J_0\rightarrowcircle J_1\rightarrow \cdots \rightarrow J_s$ in $\mb{M}$. And there exists a collider path $X(=F_0)\leftrightarrow F_1\leftrightarrow\cdots \leftrightarrow F_{t-1}\leftarrowast V$ in $\mb{M}$ with edges $F_i\rightarrow J_1,0\leq i\leq n-1$.
\end{lemma}
\begin{proof}
According to Def.~\ref{def:T1 and T2}, there exists a collider path $X(=F_0)\leftrightarrow F_1\leftrightarrow\cdots\leftrightarrow F_{t-1}\leftarrowast V$ in $\mb{M}$, where $F_1,\cdots,F_{t-1}\in \mf{W}$. There cannot be an edge $F_i\leftarrowast J_1$ in $\mb{M}$ for any $0\leq i\leq t-1$, for otherwise $J_1\in \mf{W}\cup \bar{\mf{W}}$.

Since $V$ is not an ancestor of $Y$ in $\ml{M}$ and $p$ is a minimal possible directed path, there must be $F_{t-1}\leftrightarrow V\leftarrowast J_1$ in $\ml{M}$. Hence $F_{t-1}$ is adjacent to $J_1$, for otherwise there is a new unshielded collider in $\ml{M}$ relative to $\mb{M}$. Since (1) for each $F_i,0\leq i\leq t-1$, there cannot be $F_i\leftarrowast J_1$ in $\mb{M}$, and (2) the balanced property is fulfilled in $\mb{M}$, we can conclude that there is there is $F_i\rightarrow J_1$ or $ F_i\rightarrowcircle J_1$,$\forall 1\leq i\leq n-1$ and $X\rightarrow J_1$, otherwise there is always a discriminating path for $V$ which leads to a non-circle mark at $V$ on the edge between $V$ and $J_1$ in $\ml{P}$. Due to $V\rightarrowast F_{n-1}\rightarrow J_1$ in $\mb{M}$ and the balanced property of $\mb{M}$, there is $V\rightarrowcircle J_1$. Since the path $p$ is a minimal possible directed path, the path can only be as $V\rightarrowcircle J_1\rightarrow \cdots\rightarrow J_s$. 
\end{proof}

\begin{lemma}
\label{lemma: S are ancestors}
Given a maximal local MAG $\mb{M}$, suppose a MAG $\ml{M}$ consistent with $\mb{M}$ such that there exists an adjustment set relative to $(X,Y)$. Let $\mf{W}$ be ${\normalfont \mbox{D-SEP}}(X,Y,\ml{M}_{\utilde{X}})$. For any $S$ incorporated into the set of vertices $\mf{S}$ in the process of Alg.~\ref{alg:get S} (on Line 7), there is $S\in{\normalfont \mbox{Anc}}(Y,\ml{M})$.
\end{lemma}
\begin{proof}
Note the graph $\mb{M}$ is updated in every round of Alg.~\ref{alg:get S}. To distinguish them, we use $\mb{M}_i$ to denote the graph obtained after Line 5 of Alg.~\ref{alg:get S} in the $i$-th round. Denote $\mb{M}$ the original maximal local MAG. Note in the whole process, there are no new tails introduced. Hence, for any $i$, ${\normalfont \mbox{Anc}}(Y,\mb{M}_i)={\normalfont \mbox{Anc}}(Y,\mb{M})$.

For $S\in \mf{W}$, since $\mf{W}={\normalfont \mbox{D-SEP}}(X,Y,\ml{M}_{\utilde{X}})$, there is $S\in{\normalfont \mbox{Anc}}(\mf{W}\cup\{Y\},\mb{M})$ and $S\in {\normalfont \mbox{Anc}}(Y,\ml{M})$. For $S\in \mf{S}_0$ defined in Def.~\ref{def:S_V for V in barW}, according to Lemma~\ref{lemma: possDe of barW}, since $S$ is the vertex adjacent to a vertex $V\in\bar{\mf{W}}$ in a minimal possible directed path from $V$ to a vertex in ${\normalfont \mbox{Anc}}(Y,\mb{M})$, there is $S\in {\normalfont \mbox{Anc}}(Y,\mb{M})$. Suppose there are $T$ rounds in Alg.~\ref{alg:get S}, where in the $i$-th round, $1\leq i\leq T-1$, an edge $A\leftarrowast S_i$ is transformed by $\ml{R}_{12}$ on Line 6 of Alg.~\ref{alg:get S}, and thus $\text{Anc}(S_i,\mb{M}_i)\cap \mf{T}$ is incorporated to $\mf{S}$ on Line 7 of Alg.~\ref{alg:get S}. For brevity, denote $\mf{T}={\normalfont \mbox{PossDe}}(\bar{\mf{W}},\mb{M}[-\mf{S}])\backslash \bar{\mf{W}})$ like Line 2 of Alg.~\ref{alg:get S}.

We first prove $S_1\in {\normalfont \mbox{Anc}}(Y,\ml{M})$. Since $\ml{R}_{12}$ is triggered, there is an unbridged path $p=\langle K_1,\cdots,K_m\rangle$ relative to $\mf{S}_0$ and there exists an uncovered possible directed path $\langle V,S_1,\cdots,K_j \rangle$ for $1\leq j\leq m$ in $\mb{M}_1$, where $V\in\bar{\mf{W}}$. Without loss of generality, suppose $p$ is an unbridged path. In the unbridged path, there is $C\rightcircleast K_1\leftrightcircle \cdots\leftrightcircle K_m\leftcircleast D$ in $\mb{M}$, where $C\in \ml{F}_{K_1}\backslash \ml{F}_{K_2}$, $D\in \ml{F}_{K_{m}}\backslash \ml{F}_{K_{m-1}}$, $\ml{F}_{K_i}=\{V\in \mf{V}'\mid V\rightcircleast K_i \mbox{ or }V\rightarrowast K_i\mbox{ in }\mb{M}_1\}$. Next we prove there is not $K_1\rightarrow C$ in $\mb{M}_1$. Suppose $K_1\rightarrow C$ in $\mb{M}_1$ for contradiction. As in the process of Alg.~\ref{alg:get S} we never add a tail, there is $K_1\in \rightarrow C$ in $\mb{M}$. In this case, there is a directed path $S_1\rightarrow \cdots \rightarrow K_1\rightarrow C$ in $\mb{M}$ and $C\in \mf{S}_0\subseteq {\normalfont \mbox{Anc}}(\mf{W}\cup\{Y\},\mb{M})$. Hence $S_1\in {\normalfont \mbox{Anc}}(\mf{W}\cup\{Y\},\mb{M})$. Since $V\in\bar{\mf{W}}$ is adjacent to $S_1$, $S_1$ should belong to $\mf{S}_0$, contradiction. Hence there cannot be an edge $K_1\rightarrow C$ in $\mb{M}_1$.

Hence, in $\ml{M}$ consistent with $\mb{M}$, for each vertex $F_j,1\leq j\leq m$, it is easy to prove that $F_j$ is an ancestor of either $C$ or $D$, for otherwise there will be an unshielded collider in the path $\langle C,K_1,\cdots,K_m,D\rangle$ in $\ml{M}$. And since $C,D\in \mf{S}_0$, and $C,D$ are ancestors of $Y$ in $\ml{M}$, it holds $C,D\in {\normalfont \mbox{Anc}}(Y,\ml{M})$. And since there is $V\rightarrowcircle S_1\rightarrow\cdots\rightarrow K_j$ in $\mb{M}$, $S_1$ is an ancestor of $Y$ in $\mb{M}$. Thus $S_1\in {\normalfont \mbox{Anc}}(Y,\ml{M})$. It is evident that all the vertices in $\text{Anc}(S_1,\mb{M}_1)\cap \mf{T}$ are ancestors of $Y$.

Next we prove the induction result. Suppose in the first $i$ round, each vertex in $\mf{S}_0,\text{Anc}(S_1,\mb{M}),\cdots,\text{Anc}(S_i,\mb{M})$ is an ancestor of $Y$ in $\ml{M}$. We will prove $S_{i+1}\in \text{Anc}(Y,\ml{M})$.

Since $\ml{R}_{12}$ is triggered,there is an unbridged path $p=\langle T_1,\cdots,T_f\rangle$ relative to $\mf{S}_0\cup \bigcup_{1\leq q\leq i}(\text{Anc}(S_{q},\mb{M}_{q})\cap \mf{T})$ and there exists an uncovered possible directed path $\langle V,S_{i+1},\cdots,T_s \rangle$ for $1\leq s\leq f$ in $\mb{M}_{i+1}$ where $V\in\bar{\mf{W}}$. Without loss of generality, suppose $p$ is an unbridged path. In the unbridged path, there is $J\rightcircleast T_1\leftrightcircle \cdots\leftrightcircle T_f\leftcircleast K$ in $\mb{M}_{i+1}$, where $J\in \ml{F}_{T_1}\backslash \ml{F}_{T_2}$, $K\in \ml{F}_{T_{f}}\backslash \ml{F}_{T_{f-1}}$. Next we prove there is not $T_1\rightarrow J$ in $\mb{M}_{i+1}$. Suppose $T_1\rightarrow J$ in $\mb{M}$ for contradiction.  Note in the whole process of Alg.~\ref{alg:get S}, we never add a tail, hence there is $T_1\rightarrow J$ in $\mb{M}$. In this case, there is a directed path $S_{i+1}\rightarrow \cdots \rightarrow T_1\rightarrow J$ in $\mb{M}$ and $J\in \mf{S}_0\cup \bigcup_{1\leq q\leq i}(\text{Anc}(S_{q},\mb{M}_{q})\cap \mf{T})$. If $J\in \mf{S}_0$, then $S_{i+1}\in {\normalfont \mbox{Anc}}(\mf{W}\cup\{Y\},\mb{M})$, $S_{i+1}$ should belong to $\mf{S}_0$, contradiction. If $S_{i+1}\in \text{Anc}(S_{q},\mb{M}_{q})\cap \mf{T},1\leq q\leq i$, then $S_{i+1}$ should have been incorporated into $\mf{S}$ in the $q$-round of Alg.~\ref{alg:get S}, contradiction. Hence there cannot be an edge $T_1\rightarrow J$ in $\mb{M}_{i+1}$.

Hence, in $\ml{M}$ consistent with $\mb{M}$, for each vertex $T_s,1\leq s\leq f$, it is an ancestor of either $J$ or $K$, for otherwise there will be an unshielded collider in the path $\langle J,T_1,\cdots,T_f,K\rangle$ in $\ml{M}$. And since $J,K\in \mf{S}_0\cup \bigcup_{1\leq q\leq i}(\text{Anc}(S_{q},\mb{M}_{q})\cap \mf{T})$, and $J,K$ are ancestors of $Y$ in $\ml{M}$, it holds $J,K\in {\normalfont \mbox{Anc}}(Y,\ml{M})$. And since there is $V\rightarrowcircle S_{i+1}\rightarrow\cdots\rightarrow T_s$ in $\mb{M}$, $S_{i+1}$ is an ancestor of $Y$. Thus $S_{i+1}\in {\normalfont \mbox{Anc}}(Y,\ml{M})$. Hence all the vertices in $\text{Anc}(S_{i+1},\mb{M}_{i+1})\cap \mf{T}$ are ancestors of $Y$ in $\ml{M}$.

By induction, we can prove that all the incorporated vertices in $\mf{S}$ are ancestors of $Y$ in $\ml{M}$.
\end{proof}
\begin{proof}[Theorem.~\ref{thm:existenceDSEP rules}]
Prop.~\ref{thm:existenceDSEP} has implied that $\mf{W}={\normalfont \mbox{D-SEP}}(X,Y,\ml{M}_{\utilde{X}})$ is a potential adjustment set. We will prove that a set of vertices $\mf{S}$ will be returned by Alg.~\ref{alg:get S}.

Note the graph $\mb{M}$ is updated in every round of Alg.~\ref{alg:get S}. To distinguish them, we use $\mb{M}_i$ to denote the graph obtained after Line 5 of Alg.~\ref{alg:get S} in the $i$-th round, use $\mf{S}_i$ to denote the set of vertices obtained after Line 7 of Alg.~\ref{alg:get S} in the $i$-th round. Denote $\mb{M}$ the original maximal local MAG. For $S\in \mf{S}_0$ in Def.~\ref{def:S_V for V in barW}, there must be $V\leftarrowast S$ in $\ml{M}$ if there is $V\leftcircleast S$ in $\mb{M}$. Hence the arrowheads introduced in the first round of Alg.~\ref{alg:get S} must exist in $\ml{M}$. And due to the soundness of $\ml{R}_{12}$ by Thm.~\ref{Thm:rule 12 all vertices}, all the arrowheads introduced in Alg.~\ref{alg:get S} exist in $\ml{M}$. 

In the whole process, there are no new tails introduced. Hence, for any $i$, ${\normalfont \mbox{Anc}}(Y,\mb{M}_i)={\normalfont \mbox{Anc}}(Y,\mb{M})$. For brevity, denote $\mf{T}={\normalfont \mbox{PossDe}}(\bar{\mf{W}},\mb{M}[-\mf{S}])\backslash \bar{\mf{W}}$ like Line 2 of Alg.~\ref{alg:get S}. Suppose there are $J$ rounds in Alg.~\ref{alg:get S}, where in the $i$-th round, $1\leq i\leq J-1$, an edge $A\leftarrowast S_i$ is transformed by $\ml{R}_{12}$ on Line 6 of Alg.~\ref{alg:get S}, and thus $\text{Anc}(S_i,\mb{M}_i)\cap \mf{T}$ is incorporated to $\mf{S}$ on Line 7 of Alg.~\ref{alg:get S}. Hence there is evidently $\mf{S}_{i+1}=\mf{S}_i\cup (\text{Anc}(S_{i+1},\mb{M}_{i+1})\cap \mf{T})$, for $0\leq i\leq J-1$.

It suffices to show that in the $i$-th round $1\leq i\leq J$, there is (1) ${\normalfont \mbox{PossDe}}(\bar{\mf{W}},\mb{M}_i[-\mf{S}_i])\cap {\normalfont \mbox{Pa}}(\mf{S}_i,\mb{M}_i)=\emptyset$, (2) $\mb{M}_i[\mf{S}_V]$ is a complete graph for any $V\in \bar{\mf{W}}$, where $\mf{S}_V=\{V'\in \mf{S}_i|V\leftcircleast V'\mbox{or }V\leftarrowast V'\mbox{ in }\mb{M}_i\}$, and (3) there is not an unbridged path relative to $\mf{S}_i$ in $\mb{M}_i[{\normalfont \mbox{PossDe}}(\bar{\mf{W}},\mb{M}_i[-\mf{S}_i])]$. As when these three conditions are satisfied in each round, Alg.~\ref{alg:get S} could output a set of vertices. Suppose in the $i$ round, the algorithm output ``No''. According to Alg.~\ref{alg:get S}, at least one of the three conditions is violated. We will prove the impossibility of the violations of the three conditions in the following.

If ${\normalfont \mbox{PossDe}}(\bar{\mf{W}},\mb{M}_i[-\mf{S}_i])\cap {\normalfont \mbox{Pa}}(\mf{S}_i,\mb{M}_i)\not=T$, suppose there is a minimal possible directed path $p$ from $V\in \bar{\mf{W}}$ to $T$ in $\mb{M}_i[-\mf{S}_i]$ such that each non-endpoint does not belong to $\bar{\mf{W}}$, and there is an edge $T\rightarrow S$ in $\mb{M}_i$ for $S\in \mf{S}_i$. Hence $T\in \text{Anc}(S,\mb{M}_i)$. According to Lemma~\ref{lemma: possDe of barW}, $p$ is $V\rightarrowcircle \rightarrow \cdots \rightarrow T$ in $\mb{M}$. In this case, if $S\in \mf{S}_0$, according to Def.~\ref{def:S_V for V in barW} and Lemma~\ref{lemma: possDe of barW}, there is $S\in \text{Anc}(\mf{W}\cup \{Y\},\mb{M})$. Thus $T\in \text{Anc}(\mf{W}\cup \{Y\},\mb{M})$, which implies that there is a minimal possible directed path $p'$ from $V\in \bar{\mf{W}}$ to $T\in \text{Anc}(\mf{W}\cup \{Y\},\mb{M})$ such that $p'$ is a sub-path of $p$. However, according to Def.~\ref{def:S_V for V in barW}, in the case above, the vertex adjacent to $V$ in $p'$ should belong to $\mf{S}_0$, contradicting with the fact that the path $p$ is in $\mb{M}_i[-\mf{S}_i]$. If $S_j$ is incorporated into $\mf{S}$ in Alg.~\ref{alg:get S} in the $j,j<i$ round, since there is $T\in \text{Anc}(S,\mb{M})$, $T$ should belong to $\mf{S}_{j+1},\mf{S}_{j+2},\cdots,\mf{S}_{i}$, contradicting with the fact that the path $p$ is in $\mb{M}_i[-\mf{S}_i]$. Hence there is always a contradiction if there is ${\normalfont \mbox{PossDe}}(\bar{\mf{W}},\mb{M}_i[-\mf{S}_i])\cap {\normalfont \mbox{Pa}}(\mf{S}_i,\mb{M}_i)\not=\emptyset$.

Since $\mf{W}={\normalfont \mbox{D-SEP}}(X,Y,\ml{M}_{\utilde{X}})$, for any $V\in \bar{\mf{W}}$, $V$ is not an ancestor of $Y$ in $\ml{M}$. If $\mb{M}_i[\mf{S}_V]$ is not a complete graph for some $V\in \bar{\mf{W}}$, there must be an edge $V\rightarrow S$ in $\ml{M}$, for otherwise there will be new unshielded collider at $V$. Due to Lemma~\ref{lemma: S are ancestors}, $S$ is an ancestor of $Y$. Thus $V$ is an ancestor of $Y$, thus $V\in \mf{W}\cap \bar{\mf{W}}$, contradicting with $\mf{W}\cap \bar{\mf{W}}=\emptyset$.

Finally, we prove that there is not an unbridged path relative to $\mf{S}_i$ in $\mb{M}_i[{\normalfont \mbox{PossDe}}(\bar{\mf{W}},\mb{M}_i[-\mf{S}_i])]$. Suppose $\mb{M}_i[{\normalfont \mbox{PossDe}}(\bar{\mf{W}},\mb{M}[-\mf{S}_i])]$ is not bridged relative to $\mf{S}_i$ in $\mb{M}$. Since $\mb{M}_i[{\normalfont \mbox{PossDe}}(\bar{\mf{W}},\mb{M}_i[-\mf{S}_i])]$ is not bridged relative to $\mf{S}_i$ in $\mb{M}_i$, without loss of generality, suppose an unbridged path $K_1\leftrightcircle \cdots \leftrightcircle K_m$ relative to $\mf{S}_i$, and there is $A\in \bar{\mf{W}}$. According to Lemma~\ref{lemma: S are ancestors}, all the vertices in $\mf{S}_i$ are ancestors of $Y$ in $\ml{M}$. At first, we prove $K_1,\cdots,K_m$ are ancestors of $Y$ in $\ml{M}$. As there are vertices $S_1,S_2\in \mf{S}_i$ such that $S_1\in \ml{F}_{K_1}\backslash \ml{F}_{K_2}$ and $S_2\in \ml{F}_{K_m}\backslash \ml{F}_{K_{m-1}}$, where $\ml{F}_V=\{V'\in \mf{S}_i|V\leftcircleast V' \mbox{ or } V\leftarrowast V' \mbox{ in }\mb{M}_i\}$ as Def.~\ref{def:unbridged path}. 

If $S\in \mf{S}_0$, according to Def.~\ref{def:S_V for V in barW}, there is a minimal possible directed path $p=\langle V',S,\cdots,T\rangle$ from $V'\in \bar{\mf{W}}$ to $T\in \text{Anc}(\mf{W}\cup \{Y\},\mb{M})$ where each non-endpoint does not belong to $\bar{\mf{W}}$ in $\mb{M}_i$. According to Def.~\ref{def:T1 and T2}, there exists a collider path $X\leftrightarrow F_1 \leftrightarrow\cdots \leftrightarrow F_{t-1}\leftarrowast V'$ where $F_1,\cdots,F_{t-1}\in \mf{W}$. And according to the result (2) above, $\mb{M}_i[\mf{S}_{V'}]$ is a complete graph. Hence $F_{t-1}$ is adjacent to $S$. In this case if there is an edge $F_{t-1}\leftarrowast S$ in $\mb{M}_i$, there is $F_{t-1}\leftarrowast S$ in $\mb{M}$, thus there is $S\in \mf{W}\cup \bar{\mf{W}}$. And due to $S\not\in \bar{\mf{W}}$, there is $S\in \mf{W}$. And if there is an edge $F_{t-1}\leftcircleast S$ or $F_{t-1}\rightarrow S$ in $\mb{M}_i$, there is $F_{t-1}\leftcircleast S$ or $F_{t-1}\rightarrow S$ in $\mb{M}$, there must be $V\rightarrowcircle S$ in $\mb{M}$. Due the $p$ is a minimal possible directed path in $\mb{M}_i$, $p$ is also a minimal possible directed path in $\mb{M}$, $p$ is $V'\rightarrowcircle S\rightarrow \cdots \rightarrow T$ in $\mb{M}$, thus $S\in \text{Anc}(\mf{W}\cup \{Y\},\mb{M})$. Hence, no matter what the edge is between $S$ and $F_{t-1}$, there is $S\in \text{Anc}(\mf{W}\cup \{Y\},\mb{M})$. In this case, if there is $K_j\rightarrow S$ in $\mb{M}$, there is $K_j\in \mf{S}_0$ due to $S\in \mf{S}_0$, contradicting with $K_j\in {\normalfont \mbox{PossDe}}(\bar{\mf{W}},\mb{M}_i[-\mf{S}_i])$. If $S$ is incorporated into $\mf{S}$ in the $j,j<i$ round on Line 7 of Alg.~\ref{alg:get S}, $K$ is also incorporated into $\mf{S}$ in this round due to $K\in \text{Anc}(S,\mb{M})$, contradicting with $K_j\in {\normalfont \mbox{PossDe}}(\bar{\mf{W}},\mb{M}_i[-\mf{S}_i])$. Hence for any $K_j,1\leq j\leq m$ and $S\in \mf{S}_i$ in the process of Alg.~\ref{alg:get S}, there cannot be an edge $S\leftarrow K_j$.

Hence, consider the uncovered path $p_1=\langle S_1,K_1,K_2,\cdots,K_m,S_2\rangle$ in $\mb{M}_i$ where the sub-path from $K_1$ to $K_m$ is a circle path. Note the non-circle marks in $\mb{M}_i$ also exist in $\ml{M}$ due to the soundness of $\ml{R}_{12}$ according to Thm.~\ref{Thm:rule 12 all vertices}. Since $\ml{M}$ cannot have new unshielded colliders relative to $\mb{M}_{i}$, there is each vertex in $K_1,K_2,\cdots,K_m$ is an ancestor of either $S_1$ or $S_2$ in $\ml{M}$. Since $S_1$ and $S_2$ are ancestors of $Y$ according to Lemma~\ref{lemma: S are ancestors}, any vertex in $K_1,\cdots,K_m$ are ancestors of $Y$ in $\ml{M}$.

Note in the process of Alg.~\ref{alg:get S}, we only add arrowheads at $\bar{\mf{W}}$, which are not ancestors of $Y$ in $\ml{M}$, hence we will never introduce any arrowheads at $K_j,\forall 1\leq j\leq m$ in Alg.~\ref{alg:get S}. Hence in $\mb{M}_i$, the uncovered path $p_1$ is in the form of $S_1\rightcircleast K_1\leftrightcircle \cdots \leftrightcircle K_m\leftcircleast S_2$, that is, there cannot be $S_1\rightarrowast K_1$ or $K_m\leftarrowast S_2$ in $\mb{M}_{i}$.

Next, for any $K_j,1\leq j\leq m$, consider the minimal possible directed path $p=\langle A, B_j,\cdots,K_j\rangle$ from $A$ to $K_j$. Note we use notation $B$ to denote the vertex adjacent to $A$ in the minimal possible directed path from $A$ to $K_j,1\leq j\leq m$. Without loss of generality, we suppose each non-endpoint in $p$ does not belong to $\bar{\mf{W}}$, since if there is another vertex $A'\in \bar{\mf{W}}$ in $p$, we can consider $A'$ instead of $A$, it is evidently that $K_1,\cdots,K_m$ are possible descendants of $A'$ as well since there is a minimal possible directed path from $A'$ to $K_j$ and there are circle paths from $K_j$ to each vertex in $K_1,\cdots,K_m$.

Note it is possible that there are many minimal possible directed paths from $A$ to $K_j$. Next, we prove that for any $B_i,B_j,1\leq i<j\leq m$, there is either $B_i$ and $B_j$ denote the same vertex, or $B_i$ is adjacent to $B_j$, and $B_i$ is adjacent to each vertex in $\mf{S}$ in $\mb{M}$. Suppose $B_i$ is not adjacent to $B_j$ or $S\in \mf{S}$ in $\mb{M}$. Since $S$ is an ancestor of $Y$ in any MAG $\ml{M}$ consistent with $\mb{M}$ such that $\mf{W}={\normalfont \mbox{D-SEP}}(X,Y,\ml{M}_{\utilde{X}})$, and $A\in \bar{\mf{W}}$, there is $A\leftarrowast S$ in $\ml{M}$. Since there is not $A\leftarrowast B_j$ or $A\leftarrowast B_i$ in $\mb{M}$, there is either $A\rightarrow B_j$ or $A\rightarrow B_i$ in $\ml{M}$. And since $B_j$ and $B_i$ are located at minimal possible directed paths from $A$ to $K_j$ and $K_i$, respectively, there must be $A\in \text{Anc}(K_i,\ml{M})$ or $A\in \text{Anc}(K_j,\ml{M})$. Since we have proven that $K_i,K_j$ are ancestors of $Y$ in $\ml{M}$, $A$ is an ancestors of $Y$ in $\ml{M}$, in which case there is $A\in \mf{W}\cap \bar{\mf{W}}$, contradicting with $\mf{W}\cap \bar{\mf{W}}=\emptyset$ in Def.~\ref{def:T1 and T2}. Hence, for any $B_i,B_j,1\leq i<j\leq m$, there is either $B_i$ and $B_j$ denote the same vertex, or $B_i$ is adjacent to $B_j$, and $B_i$ is adjacent to each vertex in $\mf{S}$ in $\mb{M}$.

Next, we prove that for any $1\leq j\leq m$, there is not $B_j=K_j$. That is, the minimal possible directed path from $A$ to $K_j$ cannot be $A\leftcircleast K_j$ in $\mb{M}_i$. Suppose there is $A\leftcircleast K_j$ in $\mb{M}_i$. According to the result above, for any $S\in \mf{S}$, $K_j$ is adjacent to $S$. In this case, there must be $m\geq 3$, for otherwise if the unbridged path is just $K_1\leftrightcircle K_2$, suppose $j=1$, then there must be $\ml{F}_{K_1}\supseteq \ml{F}_{K_2}$, contradicting with the definition of unbridged path in Def.~\ref{def:unbridged path}. We consider the circle path $K_j\leftrightcircle K_{j+1}\leftrightcircle\cdots \leftrightcircle K_m$ in $\mb{M}$ (If $j\geq m-1$, then we consider the circle path $K_1\leftrightcircle \cdots K_j$ instead. And it is impossible that $m=3$ and $j=2$, for otherwise the path cannot be unbridged). There is $S_2\in \ml{F}_{K_m}\backslash \ml{F}_{K_{m-1}}$. Since there cannot be an edge $K_{m-1}\rightarrow S_2$ in $\mb{M}_i$ and $\mb{M}$, which we have proven before, and $S_2\in \ml{F}_{K_m}\backslash \ml{F}_{K_{m-1}}$, $S_2$ cannot be adjacent to $K_{m-1}$. Next, we can conclude that $K_{m-2}$ is not adjacent to $S_2$, for otherwise in the substructure comprised of $K_{m-2},K_{m-1},K_m,S_2$, there must be $K_m\rightarrow S_2\leftarrow K_{m-2}$ oriented by $\ml{R}_9$ in $\ml{P}$, which leads to $K_{m-2}\rightarrow S$ in $\mb{M}$ and $\mb{M}_i$, contradiction. Similarly, we can conclude that no vertices in $K_j,K_{j+1},\cdots,K_m$ is adjacent to $S$. However, we have proven that $K_j$ is adjacent to $S$, contradiction. Hence, for any $B_j,1\leq j\leq m$, $B_j\neq K_j$.

Next, we will prove that for any $B_j,1\leq j\leq m$, $B_j$ is also in the minimal possible directed path from $A$ to $K_i$, where $1\leq i\leq m$ and $i\neq j$. Without loss of generality, suppose $i>j$. Consider the minimal possible directed path $\langle A,B_{j+1},\cdots,K_{j+1}\rangle$ from $A$ to $K_{j+1}$ in $\mb{M}$. We will prove that $B_j$ is also the vertex adjacent to $A$ in a minimal possible directed path from $A$ to $K_{j+1}$. If $B_{j+1}$ and $B_j$ denote the same vertex, the result evidently holds. We just consider the case $B_{j+1}\neq B_j$. We have proven that $B_j$ is adjacent to $B_{j+1}$ before.

Note each vertex in $K_1,\cdots,K_m$ is an ancestor of $Y$ in $\ml{M}$. According to Lemma~\ref{lemma: possDe of barW}, there must be $p_1=A\rightarrowcircle B_j\rightarrow \cdots \rightarrow K_j$ and $p_2=A\rightarrowcircle B_{j+1}\rightarrow \cdots \rightarrow K_{j+1}$. Since we have proven $K_{j+1}\neq B_{j+1}$ above, $A$ cannot be adjacent to $K_{j+1}$, for otherwise $p_2$ is not a minimal possible directed path. If $B_j$ is adjacent to $K_{j+1}$ in $\mb{M}$, it is evident that there is a minimal possible directed path $\langle A,B_j,K_{j+1} \rangle$, thus $B_j$ is also the vertex adjacent to $A$ in a minimal possible directed path from $A$ to $K_{j+1}$. If $B_j$ is not adjacent to $K_{j+1}$, due to the possible directed path $B_j\rightarrow \cdots \rightarrow K_j\leftrightcircle K_{j+1}$ in $\mb{M}_i$, the corresponding path in $\mb{M}$ is also a possible directed path, there must be a minimal possible directed path $p'$ from $B_j$ to $K_{j+1}$ in $\mb{M}$ according to Lemma~\ref{lemma:mpdp}, and thus the corresponding path of $p'$ in $\mb{M}_i$ must be also a minimal possible directed path since in Alg.~\ref{alg:get S} we only transform some edge $V\leftcircleast S$ to $V\leftarrowast S$ for $V\in \bar{\mf{W}}$ and $S\in \mf{S}_i$, while each vertex in $p'$ does not belong to $\mf{S}$ since it is in $\mb{M}_i[-\mf{S}_i]$. And $A$ is not adjacent to any non-endpoint in $p'$ since $A\rightarrowcircle B_j\rightarrow \cdots \rightarrow K_{j}$ is a minimal possible directed path and $A$ is not adjacent to $K_{j+1}$. Hence, we have a new minimal possible directed path $A\rightarrowcircle B_j\bigoplus p'$, where $B_j$ is the vertex adjacent to $A$ in a minimal possible directed path from $A$ to $K_{j+1}$. Similarly, we can prove that $B_j$ is also the vertex adjacent to $A$ in a minimal possible directed path from $A$ to $K_{j+2}$. Repeat the process, we can prove that $B_j$ is also the vertex adjacent to $A$ in a minimal possible directed path from $A$ to $K_i$, for any $1\leq i\leq m$ and $i\neq j$.

Till now, we have proven that there exists a minimal possible directed path from $A$ to each vertex in $K_1,\cdots,K_m$ such that $B_j$ is the common vertex adjacent to $A$ in all paths. And it is evidently that the minimal possible directed path is an uncovered path. Hence, if $B_j\not\in \mf{S}_i$, the edge $A\leftcircleast B_j$ should be transformed by $\ml{R}_{12}$ in $\mb{M}_i$ on Line 6 of Alg.~\ref{alg:get S}, thus the algorithm will enter the next loop, contradiction. Hence there is not an unbridged path relative to $\mf{S}_i$ in $\mb{M}_i[{\normalfont \mbox{PossDe}}(\bar{\mf{W}},\mb{M}_i[-\mf{S}_i])]$.
\end{proof}

\subsection{Proof of Proposition~\ref{prop: DD-SEP}}
\label{subsec:proof of proposition dd-sep}
\begin{proof}
We prove it by mathematical induction. For each $V\in {\normalfont \mbox{DD-SEP}}(X,Y,\mb{M}_{\utilde{X}})$, we consider the minimal collider path satisfying the three conditions of Definition~\ref{def:dd-sep} in $\mb{M}$. If the length is $1$, there is $X\leftarrowast V$ in $\mb{M}_{\utilde{X}}$ since there cannot be $X\rightast V$ in $\mb{M}_{\utilde{X}}$ according to the definition of $\mb{M}_{\utilde{X}}$. If there is $V\in \text{Anc}(Y,\mb{M})$, it trivially concludes that $V\in {\normalfont \mbox{D-SEP}}(X,Y,\ml{M}_{\utilde{X}})$ according to the definition. If $\mb{M}[\ml{Q}_V]$ is not a complete graph, evidently there are at least two vertices in $\ml{Q}_V$. Suppose $S_1,S_2\in \ml{Q}_V$ are not adjacent. It is evident that there is either $V\rightarrow S_1$ or $V\rightarrow S_2$ in $\ml{M}$, otherwise there is a new unshielded collider in $\ml{M}$ relative to $\mb{M}$ and $\ml{P}$, which contradicts with that $\ml{M}$ is consistent with $\mb{M}$. Thus there is also $V\in \text{Anc}(Y,\ml{M})$ such that $V\in {\normalfont \mbox{D-SEP}}(X,Y,\ml{M}_{\utilde{X}})$. Hence if the length is $1$, there is $V\in {\normalfont \mbox{D-SEP}}(X,Y,\ml{M}_{\utilde{X}})$. Suppose the result holds when the length of the minimal collider path mentioned above is $k$. For the vertex $V$ with a minimal collider path satisfying the three conditions of Definition~\ref{def:dd-sep} whose length is $k+1$, suppose the path is comprised of $X,V_1,V_2,\cdots,V_{k+1}$. We have $V_1,V_2,\cdots,V_k\in {\normalfont \mbox{D-SEP}}(X,Y,\ml{M}_{\utilde{X}})$. For $V_{k+1}$, similar to the proof above, no matter whether $V_{k+1}\in \text{Anc}(Y,\mb{M})$ or  $\mb{M}[\ml{Q}_{V_{k+1}}]$ is not a complete graph, there is always $V_{k+1}\in \text{Anc}(Y,\ml{M})$, thus $V_{k+1}\in {\normalfont \mbox{D-SEP}}(X,Y,\ml{M}_{\utilde{X}})$ due to the collider path where each non-endpoint belongs to ${\normalfont \mbox{D-SEP}}(X,Y,\ml{M}_{\utilde{X}})$.
\end{proof}
\subsection{Proof of Corollary~\ref{cor:corsetequals}}
\label{subsec:proof of Corollary}
\begin{proof}
The proof follows Thm. 4 of~\citet{conf/icml/WangQZ23} based on Thm.~\ref{prop:existencecritical rules} and Thm.~\ref{thm:existenceDSEP rules}. Thm.~\ref{prop:existencecritical rules} and Thm.~\ref{thm:existenceDSEP rules} can ensure that by using Alg.~\ref{alg:get S} for each potential adjustment set, we can find the set of causal effects in all the DAGs represented by the MAGs consistent with $\mb{M}$. And since in PAGrules, all possible local transformation are considered on Line 5, PAGrules can return the set of causal effects in all the DAGs represented by the MAGs consistent with $\ml{P}$.
\end{proof}



\end{document}